\documentclass[final,12pt]{alt2022}

\definecolor{DarkRed}{rgb}{0.75,0,0}
\definecolor{DarkGreen}{rgb}{0,0.5,0}
\definecolor{DarkPurple}{rgb}{0.5,0,0.5}
\definecolor{DarkBlue}{rgb}{0,0,0.7}

\hypersetup{colorlinks=true, plainpages = false, citecolor = DarkBlue, urlcolor = DarkBlue, filecolor = black, linkcolor =DarkBlue}

\usepackage{times}
\usepackage{enumitem}

\usepackage{thmtools} 
\usepackage{thm-restate}

\usepackage{algorithm}
\usepackage{algpseudocode}

\usepackage{prettyref}

\usepackage{upgreek}
\usepackage[mathscr]{euscript}
\usepackage{mathtools}
\usepackage{multirow}

\usepackage{cdann_definitions}

\newcommand{\indicator}[1]{\mathbf{1}\left(#1\right)}
\newcommand{\wt}{\widetilde}

\usepackage{prettyref}
\newcommand{\pref}[1]{\prettyref{#1}}

\newcommand{\savehyperref}[2]{\texorpdfstring{\hyperref[#1]{#2}}{#2}}
\newrefformat{eq}{\savehyperref{#1}{\textup{(\ref*{#1})}}}
\newrefformat{eqn}{\savehyperref{#1}{Equation~\ref*{#1}}}
\newrefformat{con}{\savehyperref{#1}{Conjecture~\ref*{#1}}}
\newrefformat{lem}{\savehyperref{#1}{Lemma~\ref*{#1}}}
\newrefformat{def}{\savehyperref{#1}{Definition~\ref*{#1}}}
\newrefformat{line}{\savehyperref{#1}{line~\ref*{#1}}}
\newrefformat{thm}{\savehyperref{#1}{Theorem~\ref*{#1}}}
\newrefformat{corr}{\savehyperref{#1}{Corollary~\ref*{#1}}}
\newrefformat{sec}{\savehyperref{#1}{Section~\ref*{#1}}}
\newrefformat{app}{\savehyperref{#1}{Appendix~\ref*{#1}}}
\newrefformat{ass}{\savehyperref{#1}{Assumption~\ref*{#1}}}
\newrefformat{ex}{\savehyperref{#1}{Example~\ref*{#1}}}
\newrefformat{fig}{\savehyperref{#1}{Figure~\ref*{#1}}}
\newrefformat{alg}{\savehyperref{#1}{Algorithm~\ref*{#1}}}
\newrefformat{rem}{\savehyperref{#1}{Remark~\ref*{#1}}}
\newrefformat{conj}{\savehyperref{#1}{Conjecture~\ref*{#1}}}
\newrefformat{prop}{\savehyperref{#1}{Proposition~\ref*{#1}}}
\newrefformat{proto}{\savehyperref{#1}{Protocol~\ref*{#1}}}
\newrefformat{prob}{\savehyperref{#1}{Problem~\ref*{#1}}}
\newrefformat{claim}{\savehyperref{#1}{Claim~\ref*{#1}}}

\usepackage{pifont}

\newcommand{\wh}{\widehat}

\newcommand{\N}{\mathcal{N}}

\newcommand{\R}{\mathbb{R}}

\renewcommand{\varepsilon}{\epsilon}
\renewcommand{\tilde}{\wt}
\renewcommand{\hat}{\wh}

\DeclareMathOperator*{\E}{{\bf {E}}}

\DeclareMathOperator*{\Z}{\mathbb{Z}}

\newcommand{\regret}{\operatorname{Reg}}
\newcommand{\instregret}{r}

\usepackage{nicefrac}

\begin{document}

\title{Leveraging Initial Hints for Free in Stochastic Linear Bandits}




\altauthor{%
  \Name{Ashok Cutkosky} \Email{cutkosky@google.com}\\
 \Name{Chris Dann} \Email{chrisdann@google.com}\\
 \Name{Abhimanyu Das} \Email{abhidas@google.com}\\
  \Name{Qiuyi (Richard) Zhang} \Email{qiuyiz@google.com}\\
  \addr Google AI
}

\maketitle

\begin{abstract}
We study the setting of optimizing with bandit feedback with additional prior knowledge provided to the learner in the form of an initial hint of the optimal action. We present a novel algorithm for stochastic linear bandits that uses this hint to improve its regret to $\tilde O(\sqrt{T})$ when the hint is accurate, while maintaining a minimax-optimal $\tilde O(d\sqrt{T})$ regret independent of the quality of the hint. Furthermore, we provide a Pareto frontier of tight tradeoffs between best-case and worst-case regret, with matching lower bounds. Perhaps surprisingly, our work shows that leveraging a hint shows provable gains without sacrificing worst-case performance, implying that our algorithm adapts to the quality of the hint for free. We also provide an extension of our algorithm to the case of $m$ initial hints, showing that we can achieve a $\tilde O(m^{2/3}\sqrt{T})$ regret.
\end{abstract}

\section{Introduction}

A variety of problems across different disciplines involve making sequential decisions based on noisy observations, often modeled as bandit problems. Learners interacting with a sequential decision making task naturally have access to some ``hint'', side information, or prior knowledge about what might be a good action or policy for that task. This could be, for example, gleaned from previous or related tasks completed by that learner, or simply some prior domain knowledge about the task. In other cases, the hint could be an existing baseline against which an optimizer should be competitive.

For example, consider the problem of hyperparameter tuning of a neural network on a dataset. 
If  hyperparameter tuning has previously been performed on similar datasets earlier, it would be reasonable to use the optimal hyperparameters from such tasks as a hint for which hyperparameters might be good for the current task. Such offline knowledge transfer is at the heart of emerging fields of learning research, such as transfer learning \citep{torrey2010transfer}, metalearning \citep{vanschoren2018meta}, and offline RL \citep{levine2020offline}. 

In practice, such hints might be inaccurate or imprecise for the current task at hand. A natural question to ask then is whether we can design bandit algorithms that utilize such hints to improve their regret, but, at the same time, are ``robust'' to  imprecisions in the hints. That is, if the hint is ``accurate'', it obtains better regret than an algorithm that does not have access to this hint. However, if the hint is ``inaccurate'', then the algorithm should not perform much worse than an algorithm that simply ignores the hint.


More concretely, in this paper we focus on the \emph{stochastic linear bandit} problem on the unit ball \citep{lattimore2020bandit}. Standard regret bounds for this problem is $O(d \sqrt{T})$, and this is known to be unimprovable in the worst case \citep{dani2008stochastic}. To model the notion of a ``hint'', we further suppose that the algorithm is provided with a vector $h \in \R^d$ that is some (possibly imprecise) estimate of the optimal action $a_\star$, but does not know in advance how close $h$ is to $a_\star$. In this paper, we address the following question for the stochastic linear bandit setting:
Can we design a bandit algorithm that uses this hint to obtain $o(d\sqrt{T})$ regret if the hint is sufficiently accurate (i.e ``close'' to $a_\star$), but maintains the same worst-case $O(d\sqrt{T})$ guarantee, even if the hint is inaccurate?



The analogous question has been previously studied in the multi-armed bandit setting (MAB) \citep{lattimore2015pareto} and answered in the negative. Essentially, \citet{lattimore2015pareto} shows that any MAB algorithm over $k$ arms that obtains $o(\sqrt{kT})$ regret if the hint is accurate, must suffer $\omega(\sqrt{kT})$ worst case regret if the hint is inaccurate. Specifically, for some hint arm $i$, let $R_i^T$ denote the regret with respect to arm $i$, and $R^T$ denote the regret with respect to the best arm, so $R^T = \max_i R_i^T$. Then, there are matching upper and lower bounds that show that $R^T R_i^T = \Theta(k T)$, which implies an inherent tradeoff on a Pareto frontier.

\subsection{Our Contributions}

Perhaps surprisingly, in case of stochastic linear bandits where the action space forms a unit ball ($\{ a \in \mathbb R^d \colon \|a \| \leq 1\}$), we can answer our main question in the affirmative. We introduce an algorithm that obtains a dimension-independent regret of $O(\sqrt{T})$ if the hint is sufficiently accurate, while maintaining an $O(d\sqrt{T})$ worst-case regret bound, even if the hint is arbitrarily inaccurate. Specifically, we can characterize our algorithm's  performance in terms of obtaining an $R_h^T = \tilde O(\sqrt{T})$ \emph{hint-based regret}, i.e. regret with respect to the hint action, 
while maintaining a worst case $R^T = \tilde O(d \sqrt{T})$ regret with respect to the optimal action (see \pref{thm:main_single_prior_2e}). Equivalently, let $r_h$ be the instantaneous regret of playing the hint $h$, then our algorithm achieves total regret $\min(r_hT + \tilde O(\sqrt{T}) , \tilde O(d\sqrt{T}))$.

Our algorithm exploits the geometry of the action set to quickly approximate $r_h$ by playing small perturbations of $h$. These perturbations are on the order of $O(1/\sqrt{T})$ so as to ensure $O(\sqrt{T})$ hint-based regret, but are tightly controlled as to ensure fast statistical inference of $r_h$, allowing us to bound our worst-case regret by switching to any $O(d\sqrt{T})$ linear bandit algorithm. We emphasize that all upper and lower bounds in our paper hold in the regime when $\|\theta^\star\| = \Theta(1)$, which enforces a fixed signal-to-noise ratio and our main variable then becomes purely the quality of the hint.  

Furthermore, similar to \citet{lattimore2015pareto} we characterize the inherent explore-exploit tradeoff when balancing superior performance if the hint is sufficiently accurate with inferior performance when the hint is misspecified, in terms of a tight \emph{Pareto Regret Frontier} for our algorithm. Specifically, our algorithm can achieve any hint-based regret $R_h^T < \sqrt{T}$, while maintaining a worst case regret of $O(d T/R_h^T)$ (see \pref{thm:main_single_prior_pareto}), and we obtain matching lower bounds to show that this Pareto Regret Frontier cannot be improved in general (see Theorem~\ref{thm:lower_bound_pareto}). This is particularly surprising since by restricting the action set to axis-aligned actions, we can derive MAB over $d$ arms, which has the same Pareto frontier as the more general linear bandit problem when $R_h < \sqrt{T}$. This somewhat confounds the conventional wisdom that linear bandits should suffer an extra $O(\sqrt{d})$ complexity due to the larger action set.

Lastly, we also extend our analysis to the case of multiple hints, where the algorithm is now provided $m$ hints, $h_1, h_2, \ldots, h_m$. Although this is quite practical, as optimization tasks often have multiple prior tasks to learn hints from, this setup in the stochastic setting has not been studied before, according to our knowledge. We show that we can achieve hint-based regret of  $\widetilde{O}(m^{2/3}\sqrt{T})$ with respect to the best hint $h^\star$ (see \pref{lem:multihint_hintbased}) while maintaining the usual worst case  $O(d\sqrt{T})$ regret (see \pref{lem:multihint_worstcase}). To do this, we use a clever combination of MAB on the $m$ hints while simultaneously performing perturbations of each hint to estimate its instantaneous regret. To achieve the sublinear rate, we perform a careful balancing between the sample complexity of two elimination forces: the first coming from a suboptimality compared to $h^\star$ and the second coming from suboptimality compared to $a^\star$.

We summarize our contributions as follows:
\begin{itemize}
    \item Introduce novel algorithm for stochastic linear bandits on the unit ball that achieves $R_h^T = \tilde O (\sqrt{T})$ while maintaining the usual $\tilde O(d \sqrt{T})$ regret. Equivalently, our total regret is the best of both worlds: $\min(r_hT + \tilde O(\sqrt{T}) , \tilde O(d\sqrt{T}))$.
    \item Characterize a Pareto frontier of $(R_h^T, R^T)$ tradeoffs with matching upper and lower bounds, up to log factors, that $R_h^T R^T = \Theta(dT)$. Surprisingly, this matches the same Pareto frontier for MAB with $d$ arms, implying that linear bandits may not suffer inherently higher regret than MAB in some settings.
    \item Generalize our algorithm to the practical yet novel $m$-hint setting and show sublinear hint-based regret $R_{h^*}^T = \tilde O(m^{2/3}\sqrt{T})$ while maintaining usual $\tilde O(d \sqrt{T})$ regret. Equivalently, our total regret in this case is: $\min(r_{h^*} T + \tilde O(m^{2/3}\sqrt{T}), \tilde O (d\sqrt{T}))$
\end{itemize}

\subsection{Related work}

\subsubsection{Conservative Exploration}
Conservative exploration, introduced in \citep{kazerouni2016conservative, wu2016conservative}, attempts to maximize reward while keeping its performance relative to a baseline action, which is analogous to an initial hint, above a certain multiplicative threshold. While similar to our setting, the problem is typically viewed as constrained optimization and the regret bounds given are usually of the form $O(d\sqrt{T} + S)$, where $S$ is the additional cost of following the baseline constraint. In our setting, we do not enforce any constraint; instead, we want to simultaneously minimize hint-based additive regret while diverting away from playing the hint when its instantaneous regret is too large.

\subsubsection{Tuning the Learning Rate in Mirror Descent}

A standard worst-case optimal algorithm for our linear bandit problem is to employ mirror descent with an appropriate regularizer combined with a one-point gradient estimator \cite{abernethy2008competing}. Such algorithms typically have a \emph{learning rate} parameter that could in principle be tuned to the problem at hand: if $h=a_\star$, we should center the regularizer at $h$ and set the learning rate to $0$. If $h$ is far from $a_\star$, we should instead optimize the learning rate for the worst-case scenario. Thus, the problem of effectively using the hint is essentially a problem of tuning the learning rate. Unfortunately, it is extremely unclear how and if it is possible to do this without prior knowledge of some unavailable quality value such as $\|h-a_\star\|$. Although the corresponding tuning problem for the \emph{full-information} online linear optimization problem can be solved \citep{streeter2012no,cutkosky2018black, chen2021impossible}, these techniques do not extend in any obvious way to the bandit setting.

\subsubsection{Priors with Upper Confidence Bound}

A second standard algorithm for solving the stochastic linear bandit problem is to employ upper-confidence-bound (UCB) and the optimism principle \cite{abbasi2011improved}. The standard technique involves solving a regularized least-squares subproblem in order to generate a confidence ellipsoid for the true parameter $\theta^\star$. Intuitively, the regularizer in the least-squares subproblem plays a similar role to the \emph{prior} in methods based on Thompson sampling \cite{agrawal2013thompson, abeille2017linear}. Thus, a natural strategy is to have the hint $h$ inform the regularizer or the prior. For example, instead of the standard regularized least-squares problem in which the regularizer is $\frac{\lambda}{2}\|x\|^2$, we could use $\frac{\lambda}{2}\|x-h\|^2$, which would encourage the confidence ellipsoid to be centered at $h$. Unfortunately, the influence of the regularizer on the final regret bounds for UCB seems rather small: $\lambda$ only appears inside a logarithmic term. There does not appear to be any clear way to set the regularizer in such a way that there is significant improvement when $h=a^\star$ while maintaining reasonable regret for $h\ne a^\star$ (the former naively requires exponentially large $\lambda$, which disallows the latter).

\subsubsection{Model Selection}

Yet another natural approach to incorporating a hint is through \emph{model selection}. Model selection in contextual bandits is a more general problem than the one we explore here and has been intensely studied in recent years \citep{foster2019model, ghosh2020problem, pacchiano2020model, arora2021corralling, agarwal2017corralling, chatterji2020osom, bibaut2020rate, krishnamurthy2021adapting, cutkosky2021dynamic}. In its most expansive formulation, we consider a set of $K$ ``base'' bandit algorithms, each of which may or may not achieve a good regret bound. The goal is to combine these algorithms in a black-box manner so as to produce a single bandit algorithm whose regret is guaranteed to be not much more than the best possible regret we would have been able to obtain by exclusively employing any one of the base algorithms. 

To use such a result for our problem, we could consider $K=2$ algorithms: one algorithm ignores all feedback and simply plays the hint action $h$ at all time steps. The other is any standard linear bandit algorithm that ignores the hint but uses the feedback. If the $h$ is indeed the optimal action, then the first algorithm achieves $0$ regret, but otherwise has linear regret. Alternatively, the second algorithm may always guarantee $d\sqrt{T}$ regret, which is the optimal bound for $d$-dimensional linear bandits. Unfortunately, if we were to combine these algorithms using these black-box approaches there would be significant overhead. Typical bounds (e.g. \cite{agarwal2017corralling}) would yield an algorithm whose regret when the hint is correct is $\sqrt{T}$, but decays to $d^2\sqrt{T}$ when the hint is incorrect. Thus, we must pay a significant price in the worst-case regret for improved performance in the best-case. Our goal is to instead leverage the additional structure of our problem (stochasticity, linearity, unit-ball domain) to achieve improved results when the hint is correct without compromising on worst-case performance.

\subsubsection{Hints in Adversarial Context}

In the adversarial context, there has been many areas of work that tries to incorporate once-per-round hints to bound worst case regret. We note that since we work in the stochastic context, most of these results do not apply since our hints \emph{do not predict} per-round stochastic noise and the presented bounds are too weak. Nevertheless, recent works \citep{purohit2018improving, lykouris2018competitive} introduced algorithms that tradeoff a notion of consistency (competitive ratio when hint is good or perfect) and a notion of robustness (worst-case competitive ratio). These results are supplemented with lower bounds \citep{wei2020optimal} and Pareto frontier characterizations \citep{angelopoulos2019online}. Furthermore, a recent result \citep{wei2020taking} shows how to leverage hint estimators in contextual adversarial MAB, including the setting with multiple estimators.
\section{Problem Setting and Notation}

We consider the classic \emph{stochastic linear bandit learning setting}. In each round $t \in \mathbb N$, the learner chooses an action $a_t$ from the action set $\Acal = \{ a \in \mathbb R^d \colon \|a \| \leq 1\}$ and receives a reward $y_t = \langle a_t, \theta^\star \rangle + \xi_t$ where $\xi_t$ is independent $1$-sub-Gaussian noise and $\theta^\star \in \mathbb R^d$ is the unknown true parameter vector. The performance of a learner is measured by its (pseudo-) regret
\begin{align*}
    R^T = \regret(T) = \regret_{a^\star}(T) = \sum_{t = 1}^T \instregret(a^\star, a_t)
    = \sum_{t = 1}^T \langle \theta^\star, a^\star - a_t\rangle
\end{align*}
where $a^\star = \frac{\theta^\star}{\|\theta^\star\|}$ is the optimal action and $\instregret(a, a') = \langle \theta^\star, a\rangle - \langle \theta^\star, a'\rangle$ is the instantaneous regret of action $a'$ with respect to action $a$.

\paragraph{Initial action hint:} We study the problem where the learning is provided with a hint $h \in \Acal$ in the form of an action before interacting with the bandit instance. This hint is supposed to be a guess of the optimal action available through prior knowledge. The goal of the learner is to use this hint to achieve better regret when the quality of the hint is good. The quality can be measured by the regret $r_h = \instregret(a^\star, h)$ of $h$ w.r.t. the optimal action.
To assess the degree to which a learner can leverage a good hint, we look at its regret w.r.t. the hint (\emph{hint-based regret})
\begin{align*}
    R_h^T = \regret_h(T) = R(\mathbf{a}, \mathbf{h}) = \sum_{t = 1}^T \instregret(h, a_t)
    = \sum_{t = 1}^T \langle \theta^\star, h - a_t\rangle~.
\end{align*}
Note that the regret w.r.t. the best action is the sum of the total regret of the hint and the hint-based regret, $\regret(T) = T \cdot r(a^\star, h) + \regret_h(T)$. This implies that when the hint is sufficiently good, i.e. $\instregret(a^\star, h) \leq \frac{\tilde R}{T}$, the learner with hint-based regret $\tilde R \ll d \sqrt{T}$ also achieves total regret $\regret(T) \leq  \tilde R$. Note that a naive exploitative algorithm that achieves no hint-based regret simply plays $h$ in all rounds. Thought it performs well if $h \approx a^\star$, it suffers linear regret in problems where the hint is not good, i.e., $\instregret(a^\star, h) \gg \frac{d}{\sqrt{T}}$, which is undesirable. 

\paragraph{Objective:} Our goal is to devise an algorithm that is able to leverage a hint but is also robust to its quality. 
Specifically, this algorithm should simultaneously guarantee a worst-case regret rate $\regret(T) \leq R$ and a hint-based regret rate $\regret_h(T) \leq R_h$ for all hints $h$ and bandit instances (with high probability).

\paragraph{Multiple hints:} We also consider the setting where multiple initial hints $\mathcal H = \{h_1, h_2, \dots, h_m\}$ are provided. In this case, the algorithm should maintain small hint-based regret with respect to the best hint $h^\star = \argmin_{h \in \mathcal H} r(a^\star, h)$ while ensuring a worst-case regret rate $\regret(T) \leq R$.






\paragraph{Additional notation:} Our analysis makes heavy use on the geometry of the action and parameter space. We assume the action set $\Acal$ to be the unit ball and all norms are $\ell_2$, but our results can be generalized to general ellipsoids.
When the action set consists only of axis-aligned unit vectors, this setting reduces to the \emph{multi-armed bandit setting}. To that end, it is useful to define the following projection operators and recall some of their properties. For a vector $v \in \mathbb R^d$, let $P_{v}$, $P_{v}^\perp$ be projection operators onto the span of $v$ and its orthogonal complement, respectively,
\begin{align*}
    P_{v}u &= \frac{\langle u,v\rangle v}{\|v\|^2}
    &
    P^\perp_{v} u &= u - P_{v}u,
\end{align*}
where $u \in \mathbb R^d$. Also note that $\| P_v u \| \|v\| = |\langle v, u \rangle| = \|P_u v\| \|u\|$ and $\|P_{v}^\perp u \|\|v\| = \|P_u^\perp v\|\|u\|$.

\section{Regret Lower Bounds}

Algorithms that can effectively exploit prior information, while maintaining good worst-case performance have an inherent explore-exploit tradeoff throughout optimization. In this section, we look at lower bounds that show an inherent limitation on the ability of algorithms to explore and exploit well effectively at the same time. Specifically, we will show that if an algorithm has superior hint-based regret, it must lack the explorative capabilities to bound the worst-case total regret. As discussed in the introduction, for the MAB setting with $k$ arms, \citet{lattimore2015pareto} shows that the Pareto frontier is lower bounded by $R_h^T R^T = \Omega(kT)$. For linear bandit setting, we can show analogous lower bounds of $R_h^T R^T =\Omega(dT)$ via axis-aligned perturbations of a true parameter $\theta_0$. In the following, we make the bandit instance $\theta$ for regret and hint-based regret explicit by superscripts $\theta$.

\begin{restatable}{theorem}{thmparetolower}
\label{thm:lower_bound_pareto}
For any action hint $h \in \Acal \subseteq \R^{d+1}$ with $\|h\| = 1$ , horizon $T$ and learning algorithm, there is a family $\Theta$ of stochastic linear bandit instances so that the following holds. All parameters $\theta \in \Theta$ have $\Theta(1)$ norm. Let $R_h = \max_{\theta \in \Theta} \mathbb E[\regret^{\theta}_h(T)]$ be largest hint-based regret in family $\Theta$. Then there is a bandit instance $\theta \in \Theta$ such that the expected regret w.r.t. the best arm is   
\begin{align*}
    \EE[\regret^{\theta}(T)] \geq \frac{1}{4}+\frac{T}{4}\min\left(1, \frac{d-1}{R_h}\right)\geq  \Omega\left(\frac{dT}{R_h}\right)
\end{align*}
\end{restatable}
See \pref{app:lower_bound_proofs} for the complete proof.



\begin{restatable}{theorem}{thmotherlower}\label{thm:otherlower}
\label{thm:minimumhintquality}
Let $T\ge \min(d^2/2^{8}, 2^{6})$ and $\Delta^2 = \frac{d}{2^{8}\sqrt{T}}$. For any hint $h\in \R^{d+1}$ with $\|h\|=1$ and any algorithm, there is a linear bandit instance with parameter $\theta^\star \in \RR^{d+1}$ satisfying $\|\theta^\star - \frac{h}{2}\|\le \Delta$ with optimal action $a^\star = \frac{\theta^\star}{\|\theta^\star\|}$ satisfying $\|a^\star -h\|\le  4\Delta$  and $r(a^\star,h)=\langle \theta^\star,a^\star-h\rangle\le 972\Delta^2\le O(d/\sqrt{T})$  such that
\begin{align*}
    \EE[\regret^{\theta^\star}(T)]&\geq  \frac{d\sqrt{T}}{512}
\end{align*}
\end{restatable}

This theorem provides a fundamental limit on how ``good'' a hint needs to be in order for it to provide a significant advantage over a worst-case algorithm. In particular, even if the distance between the hint $h$ and the optimal action $a^\star$ is $O(\sqrt{d}/T^{1/4})$, and even simply playing the hint on its own already achieves $T \cdot \instregret(a^\star, h)\le O(d\sqrt{T})$ regret, we still cannot leverage the hint to actually improve on the worst-case bounds. That is, the hint must be very high quality in order to be useful. This lower bound also suggests how one should go about attempting to use the hint: since a hint can only be useful if $\instregret(a^\star,h)$ is small enough that playing just the hint for all $T$ rounds would already improve upon the worst-case regret, the fundamental question to answer is \emph{how good is the hint?} If we can answer this question using a small number of observations, then we would be able to either simply play the hint if it is sufficiently good, or fall back to a standard worst-case algorithm that ignores the hint otherwise.

\section{New Algorithm for Stochastic Linear Bandits with Action Hint}

In linear bandits with a given action hint, the goal of a learning algorithm is to use the hint to achieve lower regret if possible but still achieve close minimax-optimal regret even when the hint is misleading.
To what degree an algorithm should rely a hint depends on its quality as measured by its instantaneous regret $r_h = \instregret(a^\star, h)$. In fact, if we know $r_h$, then the following simple switching algorithm achieves the desired learning properties:
\begin{algorithm2e}
\SetAlgoVlined
\SetKwInOut{Input}{Input}
\SetKwProg{myproc}{Procedure}{}{}

\DontPrintSemicolon
\LinesNumbered

\Input{hint $h \in \R^d $, approximation of instantaneous regret $\hat r_h\le r_h \in \R^+$, number of rounds  $T \in \Z$, maximum regret of hint-agnostic linear bandit algorithm $R_{LB} \in \R^+$}
\eIf{$\hat r_hT  \leq R_{LB} $}{
    play $h$\;
}{
    play hint-agnostic linear bandit algorithm $\mathcal L$ \;
}
\caption{\textsc{Switch}$(h, \hat r_h, T, R_{LB})$: Simple Switching Algorithm}
\label{alg:switch}
\end{algorithm2e}

\begin{lemma}
Let $\hat r_h$ be a $\alpha$-approximation to the instantaneous regret $r_h$ of $h$, i.e., $\hat r_h \leq r_h \leq \alpha \hat r_h$ and $R_{LB} \in \mathbb R^+$ an upper-bound on the regret of the linear bandit algorithm $\mathcal L$ played for $T$ episodes, i.e., $R^T_{\mathcal L} \leq R_{LB}$. Then the worst case regret running $\textsc{Switch}(h, \hat r_h, T, R_{LB})$ for up to $T$ rounds is bounded by $\alpha  R_{LB}$.
\end{lemma}
If $\alpha$ is a constant and we employ a  minimax-optimal hint-agnostic algorithm such as OFUL \citep{abbasi2011improved} that satisfies $R^T_{\mathcal L} \leq R_{LB} = \wt O(d \sqrt{T})$ with high probability, then \textsc{Switch} always retains the minimax-rate $\wt O(d \sqrt{T})$ with high probability as well. Moreover, since the instantaneous reward estimator is a lower bound $\hat r_h\le r_h$, \textsc{Switch} also guarantees zero regret with respect to the hint. Unfortunately, \textsc{Switch} cannot be used directly since no approximation $\hat r$ to the instantaneous regret $r_h$ is known to the learner. The main idea of our algorithm is to first estimate $r_h$ up to a constant factor without incurring too much regret and then call the switch procedure. 

\subsection{Estimating Regret of Hint Action}

Our estimation procedure for $r_h = \langle a^\star - h , \theta^\star \rangle$ is based on the following decomposition:

%
%
\begin{restatable}{lemma}{actionregret}
\label{lem:regret_unitball}
Let $a \in \mathbb R^d$ be any action with $\|a\| = 1$ and $\langle a , \theta^\star \rangle \geq - \|\theta^\star\| /2$. Then the instantaneous regret of this action is bounded as
\begin{align*}
\frac{1}{2} \frac{\|P_{a}^\perp\, \theta^\star\|^2}{\|\theta^\star\|} \quad \leq \quad  
    \langle a^\star - a , \theta^\star \rangle 
    \quad \leq \quad 3 \frac{\|P_{a}^\perp\, \theta^\star\|^2}{\|\theta^\star\|}
  ~.
\end{align*} 
\end{restatable}
Instantiating this lemma with $a = h$ shows that estimating $\|\theta^\star\|$ and $\|P^\perp_{h} \theta^\star\|$ is sufficient to compute a constant-factor approximation of $r_h$, as long as $h$ is not an extremely bad hint. In the following, we will first present our procedure for estimating $\|P^\perp_{h} \theta^\star\|$ while incurring small regret with respect to both $h$ and the optimal action $a^\star$. We will also be able to estimate $\|\theta^\star\|$ by applying essentially the same procedure with $h=0$, because $P^\perp_0\theta^\star = \theta^\star$. By combining these estimates, we obtain an estimate of $\|P^\perp_{h} \theta^\star\|^2/\|\theta^\star\|$, which is in turn an estimate of $r_h$.

Since $P^\perp_h \theta^\star$ is orthogonal to the hint $h$, we gain no information about it in rounds where $h$ itself was played. We would gain the most information if we played actions that are orthogonal to $h$ but those incur very large regret when the hint $h$ is aligned with $\theta^\star$. To balance regret and information gain, we therefore play perturbations of the hint $h$
$$ \textsc{Perturb}(h, p) = \frac{h + p}{\sqrt{1 + \|p\|^2 }} $$ 
where $p \in \mathbb R^d$ is an orthogonal perturbation to $h$ ($\langle p, h\rangle = 0$). 

We will show that by choosing the size of the perturbation correctly and applying careful statistical analysis, we can incur small hint-based regret while quickly estimating $\|P_h^\perp  \theta^\star\|$. Our perturbations will be sampled from projection Gaussian distributions, defined below:

\begin{definition}
Let the projected Gaussian distribution $\N_{h}(0, {\bf I})$ be the distribution of $P_h^{\perp}g$, where $g \sim \N(0, {\bf I})$ and $P_h^{\perp}$ is the orthogonal projection onto the complement of $h$. When $h$ is a unit vector, the distribution is equivalent to $\N(0, {\bf I} - h h^\top)$. For ease of notation, we will often drop ${\bf I}$ and $\mathcal N_h(0, \Delta) := \mathcal N(0, \Delta(\mathbf{I} - \frac{h h^\top}{\|h\|^2}))$

\end{definition}

Next, we analyze the regret of playing a perturbation for a single round, both with respect to the hint $h$ (\pref{lem:prior_regret_pareto}), and also with respect to the optimal action (\pref{lem:regret_perturb}):
\begin{lemma}[Hint-based Regret of Perturbed Action on the Unit Ball]
\label{lem:prior_regret_perturb}
Let $h \in \RR^d$ a hint action with $\|h\|=1$ and let $p \in \R^d$ be any orthogonal perturbation of $h$ with  $\| p\| < \frac{1}{8}$. 
Then, the regret w.r.t. hint $h$ of the perturbed action  $a = \textsc{Perturb}(h, p)$  is bounded as
\begin{align*} 
\langle h - a, \theta^\star \rangle \leq \|\theta^\star\|\| p\|^2 + |\langle p, P^\perp_h \theta^\star\rangle| .
\end{align*}
Furthermore, if $p \sim \N_h(0, \frac{\Delta^2}{d-1})$, then with probability at least $1-\delta$,
\begin{align*}
\langle h - a, \theta^\star \rangle 
=
O \left( \|\theta^\star\|\log(1/\delta) \left(\Delta^2 + \Delta \frac{\|P_{h}^\perp \theta^\star\|}{\|\theta^\star\|\sqrt{d-1}}\right)\right)
\end{align*}
\end{lemma}

\begin{lemma}[Regret of Perturbed Hint on the Unit Ball]
\label{lem:regret_perturb}
Let $\theta^\star \in \RR^d$ be the true parameter and $h \in \RR^d$ an action hint with $\|h\|=1$ and $\langle h, \theta^\star \rangle \geq - \|\theta^\star\| / 4$.
Further, let $p \in \R^d$ be any orthogonal perturbation of $h$ with $\| p\| < \frac{1}{8}$.
Then, the instantaneous regret of the perturbed prior action  $a = \textsc{Perturb}(h, p)$ is bounded as
\begin{align*} 
\langle a^\star - a, \theta^\star \rangle \leq  12\frac{\|P_h^\perp \theta^\star\|^2}{\|\theta^\star\|} + 3\|\theta^\star\|\| p\|^2~.
\end{align*}
\end{lemma}

 Now, we present our algorithm to estimate the norm $\|P_h^\perp \theta^\star\|$ based on perturbing the hint (\pref{alg:estimation}. First, we present an algorithm that yields an accurate approximation with some moderate constant probability of failure while maintaining small regret. Our procedure requires playing a fixed underlying action for the perturbation to obtain optimal sample complexity. Later, we will use a median of means approach to amplify this constant probability provide a high-probability guarantee (\pref{alg:hp_estimation2e}). The analysis for the constant-probability algorithm is provided in \pref{lem:estimation_procedure} and \pref{lem:regretlowprob}, while the analysis for the high-probability algorithm is provided in \pref{lem:estimatenormhp2e}.

\begin{algorithm2e}
\SetAlgoVlined
\SetKwInOut{Input}{Input}
\SetKwProg{myproc}{Procedure}{}{}

\DontPrintSemicolon
\LinesNumbered
%

\Input{reference action $h \in \R^d $, perturbation magnitude  $\Delta \in \R^+$} 
Initialize $n = 0$\;
Sample perturbation $p \in \RR^d$ from 
$\mathcal N_h(0, \frac{\Delta^2}{d'}) := \mathcal N(0, \frac{\Delta^2}{d'}(\mathbf{I} - \frac{h h^\top}{\|h\|^2}))$ where $d' = d - \indicator{h \neq 0}$\;
 \myproc{\textsc{PlayAndUpdate}$()$}{
 Increment $n \gets n + 1$\;
  Play hint $h$, observe reward $y_n$ and update  average $\overline{y}_n = \frac{1}{n} \sum_j y_j$\;

  Play $\frac{h + p}{\sqrt{\|p\|^2 + \|h\|^2}}$, observe reward $z_n$ 
   and update  average $\overline{z}_n = \frac{1}{n} \sum_j z_j$\;
   
   Compute average difference $\overline{x}_n = \overline{z}_n \sqrt{\|p\|^2 + \|h\|^2} - \overline{y}_n$\;
   
    Compute confidence width $b_n = \sqrt{\frac{3(1 + \|p\|^2 + \Delta^2)\ln(40 \ln(2n))}{n}}$\;
   
   
    \If{$|x_n| \geq 2 b_n$}{ \Return $\frac{\sqrt{d'}}{\Delta} |\overline{x}_n|$}
   
 }
\caption{\textsc{EstimateNorm}$( h,\Delta)$: Low Regret $\ell_2$-norm Estimation}
\label{alg:estimation}
\end{algorithm2e}

\begin{restatable}{lemma}{estimationproc}\label{lem:estimation_procedure}
With probability at least $0.7$,  \pref{alg:estimation} (\textsc{EstimateNorm}$(h, \Delta)$) returns $r$, a constant-factor approximation to $\|P^\perp_h \theta^\star\|$, i.e., $0.06 \|P^\perp_h \theta^\star\| \leq r \leq 5 \|P^\perp_h \theta^\star\|$, after
$\tilde O \left( \frac{d(1 + \Delta^2 + \|h \|^2)}{\Delta^2 \| P_h^\perp \theta^\star\|^2}
\right)$ calls to its \textsc{PlayAndUpdate} procedure and also satisfies $\|p\|\le 3\Delta$. 
\end{restatable}

\begin{restatable}{lemma}{regretlowprob}\label{lem:regretlowprob}
In the event where \pref{lem:estimation_procedure} holds, if $\Delta \leq 1/24$ and $\langle h, \theta^\star\rangle \geq -\frac{\|\theta^\star\|}{4}$, then the total regret incurred by \pref{alg:estimation} is
\begin{align*}
    \tilde O \left( 
    \frac{d}{\Delta^2 \|  \theta^\star\|}
    + 
    \frac{d  \| \theta^\star\|}{\| P_h^\perp \theta^\star\|^2}
    \right)
\end{align*}
and the regret w.r.t. hint vector $h$ is
\begin{align*}
\tilde O \left(  
    \frac{\sqrt{d}}{\Delta \| P^\perp_h \theta^\star\|}
    + 
    \frac{d \| \theta^\star\|}{\| P_h^\perp \theta^\star\|^2} \right)~.
\end{align*}
 \end{restatable}

To gain some intuition for these lemmas, let us look ahead a bit: we will eventually wish to set $\Delta^2=\Theta(\|\theta^\star\|/\sqrt{T})$. Further, consider the case that $\|P^\perp_h\theta^\star\|=\Theta(d\|\theta^\star\|/\sqrt{T})$ - as suggested by the lower bound \pref{thm:otherlower}, this is a ``transition case'' beyond which we may be able to leverage the hint. With these settings, \pref{lem:regretlowprob} implies that the regret is $\tilde O(\sqrt{T})$, so that we are able to estimate $\|P^\perp_h\theta^\star\|$ well (as implied by \pref{lem:estimation_procedure}) without incurring significant regret. 

The next lemma amplifies these results to a high-probability guarantee by running $O(\log(1/\delta))$ parallel instances of the constant-probability estimation and taking medians of the returned estimates. It is the main work-horse of our final algorithm analysis (see \pref{alg:hp_estimation2e} in Appendix).

\begin{restatable}{lemma}{estimatenormhp}\label{lem:estimatenormhp2e}
With probability at least $1 - \delta$, \pref{alg:hp_estimation2e} (\textsc{EstimateNormHP}$(h,\Delta, \delta)$) returns $r$, a constant-factor approximation to $\|P^\perp_h \theta^\star\|$, i.e., $0.06 \|P^\perp_h \theta^\star\| \leq r \leq 5 \|P^\perp_h \theta^\star\|$, after at most  
$n\le\tilde O \left( \frac{d(1 + \Delta^2 + \|h \|^2)}{\Delta^2 \| P_h^\perp \theta^\star\|^2}
\right)$ calls to its \textsc{PlayAndUpdate} procedure. Further, if $\Delta \leq 1/224$, then the regret incurred with respect to the hint $h$ is:
\begin{align*}
     O\left( n \|\theta^\star\| \Delta^2 + n \frac{\Delta \|P_h^\perp \theta^\star\|}{\sqrt{d}}\right) \le \tilde O \left(  \frac{d \| \theta^\star\|}{\| P_h^\perp \theta^\star\|^2}\ln\frac{1}{\delta}+
    \frac{\sqrt{d}}{\Delta \| P^\perp_h \theta^\star\|}\ln\frac{1}{\delta}
     \right)
\end{align*}
If in addition, $\langle h, \theta^\star\rangle \geq -\frac{\|\theta^\star\|}{4}$, then the total regret incurred is
\begin{align*}
    O\left(n \frac{\|P_h^\perp \theta^\star\|^2}{\|\theta^\star\|} + n\|\theta^\star\| \Delta^2\right)\le \tilde O \left( 
    \frac{d(1 + \|h \|)}{\Delta^2 \|  \theta^\star\|}\ln \frac{1}{\delta}
    + 
    \frac{d(1 + \|h \|)  \| \theta^\star\|}{\| P_h^\perp \theta^\star\|^2 }\ln \frac{1}{\delta}
    \right)~.
\end{align*}
\end{restatable}

{\bf Remark:} Notice that during the norm estimation process, we play the \emph{same} perturbed action over and over again, as opposed to re-sampling the perturbation $p$ afresh for each action. Perhaps counterintuitively, fixing the perturbation $p$ is actually crucial for a low sample complexity in denoising the norm estimate, since resampling $p$ adds significant noise into the observations. 

To see how this might be, consider observing samples of $y \sim N(\mu, 1)$, where $\mu \sim N(0, \sigma^2)$ in one instance and $\mu = 0$ in another instance. Our goal is to distinguish which instance we are in from our sampled observations and $\sigma < 1$ is small, so it's a relatively difficult distinguishing task. In this case, we also allow the observer to resample $\mu$ from the mean distribution, if they wish, before observing the final sample.

If we fix the randomness of $\mu$, we are distinguishing between samples from $N(0, 1)$ and $N(\mu, 1)$, where $|\mu|$ is likely to be at least $\sigma$, so the total variational distance between those distribution is on the order of $\sigma$. 

However, if we vary $\mu$ and resample at each observation, then we are distinguishing between $N(0, 1)$ vs $N(0, 1 + \sigma^2)$, which is significantly harder since the total variational distance between these distribution is on the order of $\sigma^2$, not $\sigma$. 

\begin{algorithm2e}
\SetAlgoVlined
\SetKwInOut{Input}{Input}
\SetKwProg{myproc}{Procedure}{}{}

\DontPrintSemicolon
\LinesNumbered
\newcommand\mycommfont[1]{\footnotesize\textcolor{DarkBlue}{#1}}
\SetCommentSty{mycommfont}
\Input{hint $h \in \R^d$, number of rounds $T$, failure probability $\delta$, bound on constant factor worst-case regret scalings for three phases $W$.}
\tcp{Phase 1: Estimate norm $\|\theta^\star\|$}
Initialize $C_0 \gets
\textsc{EstimateNormHP}(0, 1, \frac{\delta}{4})$ (from \pref{alg:hp_estimation2e})\;
Call $C_0.$\textsc{PlayAndUpdate}$()$ until it returns a value $r$\;

\tcp{Phase 2: Estimate norm of orthogonal complement $\|P^\perp_h \theta^\star\|$}
Set exploration radius $\Delta = \frac{1}{\sqrt{r}T^{1/4}}$\;
Initialize $C_+ \gets \textsc{EstimateNormHP}(+h, \Delta, \frac{\delta}{4})$ and $C_- \gets \textsc{EstimateNormHP}(-h, \Delta, \frac{\delta}{4})$\;
Initialize active set $\mathcal S = \{ C_+, C_-\}$\;
\Repeat{any instance in $S$ returns a value $r_\perp$ satisfying $\frac{0.06}{2\cdot 5^2} \cdot\frac{r_\perp^2}{r} \geq 10 W d\log(T)/ \sqrt{T}$ or $|S| = 1$ and the lone instance returns $r_\perp$}{
\ForEach{active instance $C_i \in \mathcal S$}{
\tcp{If $C_i$ already returned a norm estimate, PlayAndUpdate no longer plays an perturbation.}
 Call $C_i.$\textsc{PlayAndUpdate}$()$\;
 
 \tcp{Maintain CI of hint's expected reward}
  $\mathcal R_i \gets$ all reward samples obtained by $C_i$ so far playing unperturbed hint\;
  Compute confidence interval $Y_i = (\bar y_i - b_i, \bar y_i + b_i)$ with $\bar y_i = \frac{1}{|\mathcal R_i|} \sum_{y \in \mathcal R_i} y$ and $b_i = \sqrt{\frac{3\ln(40 \ln(2|\mathcal R_i|) / \delta)}{|\mathcal R_i|}}$\;
}
\tcp{Eliminate worse hint if possible}
\If{$Y_+ \cap Y_- = \varnothing$}
{
Remove $C_i$ with smaller $\bar y_i$ from active set $\mathcal S$\;
}

}
\tcp{Phase 3: Commit to hint or ignore it. }

For all remaining rounds, call $\textsc{Switch}(\hat h, \frac{0.06}{2\cdot 5^2} \cdot\frac{r_\perp^2}{r}, T, Wd\log(T)\sqrt{T})$ from \pref{alg:switch} with hint $\hat h$ randomly chosen from a surviving active instance in $\mathcal S$ 

\caption{\textsc{ParetoBandit}$(h, T, \delta, W)$: Pareto-Optimal Bandit Algorithm on Unit Ball}
\label{alg:unit_ball2e}
\end{algorithm2e}

Now, we are in a position to put together our main algorithm and analysis. The algorithm has three distinct phases. The first two phases apply our low-regret norm estimation procedure (\pref{alg:hp_estimation2e}) to estimate $\|\theta^\star\|$ (by setting $h=0$) and $\|P^\perp_h\theta^\star\|$ respectively. The last phase simply combines the values to estimate $\instregret(a^\star, h)$, and then calls \pref{alg:switch}. However, there is a significant subtlety that must be overcome when estimating $\|P^\perp_h\theta^\star\|$: if $h$ is actually a very poor hint (i.e. $\langle \theta^\star, h\rangle\le-\|\theta^\star\|/4$), then \pref{alg:hp_estimation2e} will actually incur a large regret. In order to avoid this issue, we observe that at least one of $h$ and $-h$ must be positively correlated with $\theta^\star$, so that at least one of these must be usable with \pref{alg:hp_estimation2e}. Further, if either one (say $h$), satisfies $\langle \theta^\star, h\rangle \le -\|\theta^\star\|/4$, then there will be a large gap between the rewards for actions $h$ and $-h$. This means that we can quickly select which of the two is positively correlated. 

Finally, performing this two-arm selection procedure while simultaneously norm estimating allows norm estimation to succeed extremely quickly before the two-arm selection ends. To get around this, note that we only fully eliminate either $h, -h$ if it is clearly a bad choice according to the norm estimate. Otherwise, we will continue calling $\texttt{PlayAndUpdate}()$, although it is important to note that we will no longer play the perturbed hint after the instance has returned a norm estimate, so as to minimize hint-based regret. The full description is provided in the pseudocode, with analysis in \pref{thm:main_single_prior_2e} below:

\begin{theorem}[Main Regret Bound]
\label{thm:main_single_prior_2e}
Suppose we instantiate \pref{alg:switch} with any standard worst-case optimal linear bandit algorithm. Then there exists an absolute constant $W$ such that
with probability at least $1-\delta$, \pref{alg:unit_ball2e} has worst case regret at most $O(d\sqrt{T}\log(T/\delta))$. Further, if $\|\theta^\star\| \geq \max(d, 3734)/\sqrt{T}$ , then the hint based regret simultaneously satisfies $R_h^T = O(\sqrt{T}\log(T/\delta))$.
\end{theorem}

\section{Other Trade-offs on the Pareto Frontier}

Unsurprisingly, we can in fact achieve anything on the Pareto frontier of the hint-based vs total regret. Specifically, we show that we can maintain a hint-based regret of $R_h^T = O(G\log(T))$ while maintaining a worse-case regret of $R^T = O(dT\log(T)/G)$. As shown by our lower bounds, this tradeoff is in fact tight, up to log factors. First, we prove a more generic version of Lemma~\ref{lem:regretlowprob} by simply adjusting the values of $\Delta^2$ and removing the dependence on $\|P_h^\perp a\|$ via a worst-case analysis.

\begin{lemma}
Suppose $R$ is a constant factor approximation to $\|\theta^\star\|$ and $G = \|\theta^\star\|\Delta^2$. Running $\textsc{EstimateNorm}(p, \Delta g)$, with at least constant probability over random $g$, will incur a hint-based regret of 
\begin{center}
$R_h^T = {O}\left( GT\log(T) + \frac{\log(T)}{G}\right)$ 
\end{center}

Furthermore, if $\langle h, \theta^\star \rangle  \geq -\|\theta^\star\|/4$, then with probability at least $0.7$, we maintain a worst case regret of 

\begin{center}
$R^T = O\left(dGT\log(T) + \frac{d\log(T)}{G}\right)$  
\end{center}

\label{lem:prior_regret_pareto}
\end{lemma}

We note that there is a general segment in our algorithm that plays both $h, -h$ and requires running a 2-arm bandit process that chooses between $h, -h$. Specifically, we can view the $h, -h$ as the arms and each pull of the arm corresponds to playing the hint and its perturbation, encapsulated by a call to \textsc{PlayAndUpdate}. While the pull plays the perturbed action, the  returned reward ignores the perturbed action and is simply the reward of playing the hint. By modifying the bandit algorithm to favor choosing $h$ by following the hint-based MAB algorithm presented in \cite{lattimore2015pareto}, we can bound the hint-based regret and worst case regret tradeoff and surprisingly derive the same Pareto frontier upper bounds for linear bandits as for multi-armed bandits when $k = d$ (see \pref{alg:unit_ball_pareto2e}).

\begin{theorem}[Pareto Frontier]
Let $G \leq \sqrt{T}$ and $\|\theta^\star\| \geq  d/G$. Then, there is an alteration of Algorithm~\ref{alg:unit_ball2e} with constants $c_0 > c_1$ (\pref{alg:unit_ball_pareto2e} with MAB algorithm described in \cite{lattimore2015pareto}) that has hint based regret $$R_h^T = O(G\log(T))$$ and worst case regret at most $$R^T = O(dT\log(T)/G)$$
\label{thm:main_single_prior_pareto}
\end{theorem}

\section{Multiple Hints}

Suppose that we are given $m$ separate hints $h_1,..., h_m$ and our goal is to perform as well as the best hint, while maintaining standard worst-case regret bounds even when none of the hints are trustworthy. Specifically, let $h^\star$ be the best hint, then we want to bound $R_{h^*}^T$ and $R^T$. For now, $m$ should be a small polynomial power of $d$ and if $m$ gets too large, it will likely become futile to perform hint selection. 

For now, consider simply selecting the best hint in a multi-armed bandit setting, where each arm is the hint $h_1,...h_m$, and assume our loss is still linear. Note that the best known algorithms give an upper bound and matching lower bound of $R_{h^\star}^T = O(\sqrt{mT})$, up to $\log$ factors, at least in the regime when $m  = o(d)$. Since we want to perform well with respect to the best hint, our algorithm will rely on a multi-armed bandit analysis, where each arm is a separate hint. However, we cannot spend all of our sample budget on selecting among $\{h_1,...,h_m\}$ as we may incur large worst case regret since our hints may have similar performance but all of them are incur large instantaneous regret. 


Therefore, we will necessarily have to add a suboptimality estimation phase for each hint, at least until the hint is clearly not $h^\star$. A naive analysis would suggest that this may imply that $R_{h^\star}^T = O(m\sqrt{T})$ but we show surprisingly that we can achieve a sublinear rate $R_{h^\star}^T = \widetilde{O}(m^{2/3}\sqrt{T})$ while maintaining the usual worst case guarantees of $\widetilde{O}(d\sqrt{T})$. This implies that we can perform some tradeoff between the sample complexity of two elimination forces: the first coming from a suboptimality compared to $h^\star$ and the second coming from suboptimality compared to $a^\star$.

{\bf Remark:} As noted throughout the paper, we can compare to the multi-armed bandit setting, in which we have $m$ hint arms $h_1,...h_m$ that we want to perform well with respect to, as well as some additional $k-m$ arms $a_{m+1},...,a_k$. Note that the optimal action $a^\star$ may not be a hint arm. This specific problem was studied in \cite{lattimore2015pareto} and a quick calculation shows that there are matching upper and lower bounds that gives $R_{h^*}^T = \Theta(\sqrt{mT})$, while maintaining a worst case regret of $R^T = \Theta(\frac{k}{\sqrt{m}} \sqrt{T})$. Note that the product of the hint-based and worst case regret is still $\Omega(kT)$.

Throughout this section, we will simplify and focus on the main ideas by assuming that $\|\theta^\star \| = \Omega(1)$. Our main algorithm is present in the appendix (see \pref{alg:unit_ball_multi2e}). By following the algorithmic alterations in the previous section, we note that adding a $\|\theta^\star\|$ estimation step is straightforward and one might be able to apply suitable tradeoffs by using the hint-based MAB algorithm in the remark above to obtain a full Pareto frontier.

\begin{lemma}
\label{lem:multihint_worstcase}
Assume there are absolute constants $c_1, c_2 \in \RR^+$ such that $c_1 \leq \|\theta\| \leq c_2$ and assume \pref{alg:unit_ball_multi2e} is called with a hint set $\mathcal H$ that contains $-h$ for each $h \in \mathcal H$. 
With high constant probability, the regret of \pref{alg:unit_ball_multi2e} with $|\mathcal H| = m \leq d$  after $T$ rounds is bounded as
\begin{align*}
    R^T = \tilde O(d\sqrt{T}).
\end{align*}
\end{lemma}

\begin{lemma}
\label{lem:multihint_hintbased}
By setting $B = m^{1/3}$ and $c_0$ sufficiently large in \pref{alg:unit_ball_multi2e}, the hint-based regret of our algorithm is bounded with high probability by

$$R_{h^*}^T =  \tilde O(m^{2/3}\sqrt{T}) $$

\end{lemma}

\section{Conclusion}

The problem of improving best-case performance without sacrificing worst-case performance has proven to be surprisingly difficult in the bandit setting, in comparison to the full-information setting. In this paper, we provided an intriguing positive construction for the case of stochastic linear bandits over the unit ball. We show that, when provided with a sufficiently high-quality \emph{hint} for the optimal action, we are able to significantly improve the regret to only $\tilde O(\sqrt{T})$ larger than what would be incurred by simply playing the hint for all time steps, while gracefully decaying to the worst-case optimal bound of $\tilde O(d\sqrt{T})$ when the hint is only mediocre \emph{without prior knowledge of the hint quality}. We provide lower bound demonstrating optimality of our construction, but some open questions remain. Notably, our improved rates only arise when the parameter $\theta^\star$ is sufficiently large, although we show in \pref{thm:otherlower} that even with large $\theta^\star$, a hint is still required to go beyond worst-case bounds. Furthermore, the Pareto frontier is poorly understood when the actions sets have different geometry or in the multi-hint case when $m$ is significantly greater than $d$. It is our hope that the techniques presented here may shed an optimistic light on when it is possible to go beyond worst-case in further bandit problems.



\bibliography{main}
\clearpage
\appendix 
\renewcommand{\contentsname}{Contents of main article and appendix}
\tableofcontents
\addtocontents{toc}{\protect\setcounter{tocdepth}{3}}
\clearpage
\section{Proofs of Regret Lower Bounds}
\label{app:lower_bound_proofs}
\thmparetolower*

\begin{proof}
Without loss of generality, let $h = e_1$ be axis-aligned for ease of notation. 
Set $\theta_0 = \rho h$ for a $\rho \in \RR^+$ defined later and consider the following family of linear bandit instances, identified by their reward parameters,
\begin{align*}
    \Theta = \{ \theta_0 \} \cup \{ \theta_0 + \Delta e_i, \theta_0 - \Delta e_i \colon i = 2, 3, \dots, d \}~,
\end{align*}
where $\Delta \in \RR^+$ will be defined later. Further, the reward distribution for any action $a$ is a standard normal random variable centered at $\langle a,\theta\rangle$. Essentially, this family contains $\theta_0$ and parameters that deviate from it in any one dimension (except the first).
First note that since $\theta_0 \in \Theta$, we have
\begin{align}
\label{eqn:hbreg_theta0}
    \EE[\regret_h^{\theta_0}(T)] = \sum_{t=1}^T \rho (1 - \EE_{\theta_0}[a_{t1}]) \leq R_h
\end{align}  
where $a_{t1}$ is the first component of the action $a_t$ played in round $t$ by the algorithm.

Next, we consider any other instance $\theta^\star = \theta_0 \pm \Delta e_i \in \Theta$ for some $i$. Since the action space is the unit ball, the  optimal action in this instance is $$a^\star = \frac{\theta^\star}{ \|\theta^\star\|} = \frac{\theta^\star}{\sqrt{\rho^2 + \Delta^2}} = \frac{\theta^\star}{r},$$
where $r = \sqrt{\rho^2 + \Delta^2}$. The expected regret after $T$ rounds in bandit instance $\theta^\star$ is
\begin{align}
    \EE[\regret^{\theta^\star}(T)] &= \sum_{t=1}^T \EE\left[ \langle a^\star - a_t, \theta^\star \rangle \right] \nonumber\\
    &= 
    \sum_{t=1}^T \rho\left(\frac{\rho}{r} - \EE [a_{t1}] \right) + \sum_{t=1}^T   \Delta \EE\left( \frac{\Delta}{r} - a_{ti} \cdot \sign(\theta^\star_i)\right) \nonumber\\
    &= \sum_{t=1}^T \rho\left(\frac{\rho}{r} - \EE [b_{t1}] \right) + \sum_{t=1}^T  \Delta \EE\left( \frac{\Delta}{r} - b_{ti} \right) 
    \label{eqn:regret_form}
\end{align}
where $b_{t1} = a_{t1}$ and $b_{ti} = a_{ti} \cdot \sign(\theta^\star_i)$ and expectations are taken w.r.t. the bandit instance $\theta^\star$.  Notice that 
\begin{align*}
\frac{r}{2 \Delta} \left(\frac{ \Delta}{r} - b_{ti}\right)^2 &=  \left[\frac{ \Delta}{2r} - b_{ti} + \frac{r}{2 \Delta} b_{ti}^2 \right] 
\overset{(i)}{\leq} - \frac{ \Delta }{2r} +  \frac{r}{2 \Delta} (1 - b_{t1}^2) + \frac{ \Delta}{r} - b_{ti}
\end{align*}
where step $(i)$ follows from $1 \geq \|b_t\|^2 = b_{t1}^2 + \sum_{i=2}^{d} b_{ti}^2$. Now, denote $X_t := \frac{\rho}{r} -  b_{t1}$ and $Y_t := -\frac{\Delta}{2r} +  \frac{r}{2 \Delta} (1 - b_{t1}^2)$. Then

\begin{align*}
    2Y_t &= -\frac{ \Delta }{r} + \frac{r}{ \Delta}\left(1 - \left(\frac{\rho}{r} - X_i\right)^2 \right) 
    = -\frac{ \Delta }{r} + \frac{r}{\Delta} \left(1 - \frac{\rho^2}{r^2} + 2 \frac{\rho}{r}X_t - X_t^2 \right)\\
    &= \underset{=0}{\underbrace{-\frac{\Delta }{r} + \frac{r}{ \Delta} \left(1 - \frac{\rho^2}{r^2}\right)}} + 2 \frac{\rho}{\Delta}X_t - \frac{r}{\Delta}X_t^2 
    = 2 \frac{\rho}{ \Delta} X_t - \frac{r}{ \Delta}X_t^2
\end{align*}
and thus
$Y_t = \frac{\rho}{ \Delta} X_t - \frac{r}{2 \Delta}X_t^2$. Using these identities, we rewrite the regret in \pref{eqn:regret_form} as
\begin{align*}
    \EE[\regret^{\theta^\star}(T)]
    &\geq \sum_{t=1}^T \rho\EE[X_t]  + \sum_{t=1}^T  \Delta \EE\left( \frac{r}{2 \Delta}\left[ \frac{ \Delta}{r} - b_{ti} \right]^2 - Y_t\right)\\
    &= \sum_{t=1}^T \rho\EE[X_t]  + \sum_{t=1}^T \EE \left( \frac{r}{2} \left[ \frac{ \Delta}{r} - b_{ti} \right]^2 - \rho X_t + \frac{r}{2} X_t^2\right)\\
    &= \frac{r}{2} \sum_{t=1}^T\EE \left[X_t^2\right] +  \frac{r}{2}  \sum_{t=1}^T\EE\left[ \left(\frac{\ \Delta}{r} - b_{ti} \right)^2\right]\\
    & \geq \frac{r}{2}  \EE\left[\sum_{t=1}^T \left(\frac{\ \Delta}{r} - b_{ti} \right)^2\right]~.
\end{align*}
We now follow the analysis of Theorem~24.2 by \citet{lattimore2020bandit}.
Define $\tau_i = T \wedge \min\{t \colon \sum_{s=1}^t a_{si}^2 \geq T \Delta^2 /  r^2\}$, which is a stopping time and $U_i(x) = \sum_{t=1}^{\tau_i} \left(\frac{ \Delta}{r} - a_{ti} x\right)^2$ for $x \in \{-1, +1\}$. Then we can lower-bound the expression above as
\begin{align}
     \frac{r}{2}  \EE\left[\sum_{t=1}^T \left(\frac{\Delta}{r} - b_{ti} \right)^2\right]
    \geq \frac{r}{2}  \EE\left[\sum_{t=1}^{\tau_i} \left(\frac{ \Delta}{r} - b_{ti} \right)^2\right]
    = \frac{r}{2}  \EE\left[U_i(\sign \theta^\star_i)\right].\label{eqn:regretbound}
\end{align}
Let $\PP_0$ be the action distribution of the algorithm in bandit instance $\theta_0$ and $\PP_i$ be the action distribution in bandit instance $\theta^\star = \theta_0 + \Delta e_i$ up to round $\tau_i$. By Pinsker's inequality, we have
\begin{align*}
    |\EE_{\theta^\star}[U_i(1)] - \EE_{\theta_0}[U_i(1)]|  &\leq \sqrt{\frac{1}{2}D(\PP_0, \PP_i)} \sup U_i(1)
    \leq 
 \left( 4 T \frac{ \Delta^2}{r^2} + 2 \right)\sqrt{\frac{1}{2}D(\PP_0, \PP_i)} 
\end{align*}
where the second inequality follows from the following bound on $U_i(1)$ 
\begin{align*}
    U_i(1) = \sum_{t=1}^{\tau_i} \left(\frac{\Delta}{r} - a_{ti} \right)^2
    \leq 2 \tau_i \frac{ \Delta^2}{r^2} + 2 \sum_{t=1}^{\tau_i}  a_{ti}^2
    \leq 4 T \frac{ \Delta^2}{r^2} + 2.
\end{align*}
Since $\theta^\star$ and $\theta_0$ only differ on the $i$-th coordinate, and the noise in the observed rewards for any action has a standard normal distribution, we can bound the KL divergence of the output distribution of the $t$-th action as $\frac{\Delta^2}{2} a_{ti}^2$.
Now, by the chain rule, we bound
\begin{align*}
    D(\PP_0, \PP_i) \leq 
\frac{\Delta^2}{2} 
\EE_{\theta_0}\left[
\sum_{t=1}^{\tau_i}  a_{ti}^2
\right] 
\end{align*}
Consider now the index $j \in \{2, 3, \dots, d\}$ which minimizes $\EE_{\theta_0}\left[
\sum_{t=1}^{\tau_i}  a_{tj}^2
\right]$. For this index, we can further bound
\begin{align*} 
D(\PP_0, \PP_j) 
\leq 
\frac{\Delta^2}{2} \frac{1}{d-1} \sum_{i=2}^{d} \EE_{\theta_0}\left[
\sum_{t=1}^{\tau_i}  a_{ti}^2
\right] 
\leq \frac{\Delta^2}{2} \frac{1}{d-1} \sum_{t=1}^{\tau_i} \EE_{\theta_0}[ 1 - a_{t1}^2]
\leq \frac{\Delta^2}{d-1} \sum_{t=1}^{\tau_i} \EE_{\theta_0}[ 1 - a_{t1}]
\end{align*}
where the last inequality follows from $1 - x^2 \leq 2 (1- x)$ for $x \in [-1, 1]$. Finally, the RHS above can be upper-bounded by $\frac{\Delta^2}{d-1} \frac{R_h}{\rho}$  
using \pref{eqn:hbreg_theta0}.
Thus, we have shown that
\begin{align*}
    |\EE_{\theta^\star}[U_j(1)] - \EE_{\theta_0}[U_j(1)]| \leq 
    \left( 4 T \frac{ \Delta^2}{r^2} + 2 \right)\sqrt{\frac{\Delta^2 R_h}{2\rho (d-1)}}
\end{align*}
and by symmetry, we can also show that the same upper-bound holds for $|\EE_{\theta^\star_-}[U_j(-1)] - \EE_{\theta_0}[U_j(-1)]|$ where $\theta^
\star_- = \theta_0 - \Delta e_j$.
 Therefore, we have
\begin{align*}
    \EE_{\theta^\star}[U_j(1)] + \EE_{\theta^
    \star_-}[U_j(-1)] &\geq \EE_{\theta_0}[U_j(1) + U_j(-1)] - 
    2\left( 4 T \frac{ \Delta^2}{r^2} + 2 \right)\sqrt{\frac{\Delta^2 R_h}{2\rho (d-1)}}
    \\
    &= 4\EE_{\theta_0}\left[ \frac{\tau_j \Delta^2}{r^2} + \sum_{t=1}^{\tau_j} a_{ti}^2\right] -  2\left( 4 T \frac{ \Delta^2}{r^2} + 2 \right)\sqrt{\frac{\Delta^2 R_h}{2\rho (d-1)}}\\
    &\ge 2\left( 4 T \frac{ \Delta^2}{r^2} + 2 \right) -  2\left( 4 T \frac{ \Delta^2}{r^2} + 2 \right)\sqrt{\frac{\Delta^2 R_h}{2\rho (d-1)}}
\end{align*}
Now, set $\Delta^2 = \frac{(d-1)\rho}{2R_h}$, yielding:
\begin{align*}
    \EE_{\theta^\star}[U_j(1)] + \EE_{\theta^
    \star_-}[U_j(-1)]&\ge  4 T \frac{ \Delta^2}{r^2} + 2 \\
    &= 4T \frac{(d-1)\rho}{2R_h r^2} + 2
    \intertext{Use $r^2=\Delta^2+\rho^2 = \frac{(d-1)\rho + 2R_h \rho^2}{2R_h}$:}
    &=\frac{4T (d-1)\rho}{(d-1)\rho + 2R_h \rho^2} + 2
\end{align*}
Now, we need to set $\rho$ in order to both enforce the constraint $ r^2\le 1$, and also achieve our desired regret lower bounds. To this end, we consider two cases, either $d-1\le R_h$ or not. If $d-1\le R_h$, then set $\rho=\frac{1}{2}$ to obtain $r^2= \frac{\frac{d-1}{2} + \frac{R_h}{2}}{2R_h}\le \frac{d-1}{4R_h} + \frac{1}{4}\le 1/2\le1$. In this case, we also have 
\begin{align*}
    \EE_{\theta^\star}[U_j(1)] + \EE_{\theta^
    \star_-}[U_j(-1)]&\ge\frac{4T (d-1)}{(d-1)+ R_h } + 2\\
    &\ge \frac{2T(d-1)}{R_h} + 2
\end{align*}
This implies that at least one of $\EE_{\theta^\star}[U_j(1)]$ or $\EE_{\theta^\star_-}[U_j(-1)]$ is greater than $\frac{T(d-1)}{R_h} + 1$. Thus by \pref{eqn:regretbound}, there is some $\theta^\star$ such that
\begin{align*}
    \EE[\regret^{\theta^\star}(T)]&\ge \frac{r}{2}\left(\frac{T(d-1)}{R_h} + 1\right)\\
    &\ge \frac{T(d-1)}{4R_h} + \frac{1}{4}
\end{align*}
where we have used $r\ge \rho =1/2$.
    
Otherwise (if $d-1\ge R_h$), set $\rho = \frac{R_h}{2(d-1)}$ to obtain $r^2 = \frac{R_h}{2(d-1)}$ to obtain $r^2 = \frac{R_h/2 + R_h^3/2(d-1)^2}{2R_h}\le \frac{1}{2} + \frac{R_h^2}{2(d-1)}\le 1$ again. Further, this also implies the upper bound $r^2\ge 1/4$. Thus, we now have:
\begin{align*}
    \EE_{\theta^\star}[U_j(1)] + \EE_{\theta^
    \star_-}[U_j(-1)]&\ge \frac{2TR_h }{R_h/2 + R_h^3/2(d-1)^2} + 2\\
    &\ge 2T+2
\end{align*}
This implies that there is some $\theta^\star$ for which:
\begin{align*}
    \EE[\regret^{\theta^\star}(T)] &\ge \frac{r}{2}\left(T + 1\right)\\
    &\ge \frac{T}{4} + \frac{1}{4}
\end{align*}
Putting both cases together provides the desired lower bound.

\end{proof}

\begin{lemma}\label{lem:noncenteredlowerbound}
Let $\theta_0 \in \RR^{d+1}$, $T \in \NN$ and $\Delta \in (0, \sqrt{3/4})$. Then for any algorithm, there is a linear bandit instance with unit-ball action set ($\Acal = \{ a \in \RR^{d+1} \colon \|a\|_2 \leq 1\}$ ) and parameter $\theta^\star \in \RR^{d+1}$ with $\|\theta^\star - \theta_0\|_2 = \Delta$ such that the expected regret after $T$ rounds is at least 
\begin{align}
    \EE[\regret_{\theta^\star}(T)] &\geq  \frac{T\Delta}{2\sqrt{1 + \frac{\|\theta_0\|^2}{\Delta^2}}} - \sqrt{d} \Delta\frac{\sqrt{\Delta^2 + \|\theta_0\|^2}}{8}
    \left( \frac{4T}{\frac{d\|\theta_0\|^2}{\Delta^2} + d} + 2 \right)\sqrt{\frac{T}{\frac{d\|\theta_0\|^2}{\Delta^2} + d} + 1}
    \\
    &\geq \frac{T \Delta^2}{2} - \frac{\sqrt{d} \Delta}{4}
    \left( \frac{32}{d} T \Delta^2 + 1 \right)^{3/2}\label{eqn:lb_regret}
\end{align}
where the second form holds when $\Delta \leq \sqrt{\frac{3}{4}}$ and $\|\theta_0\| = 1/2$.
\end{lemma}

\begin{proof}
Let $\bar \Delta = \frac{\Delta}{\sqrt{d}}$.
Without loss of generality, let $\theta_0 = \theta e_1$ be axis-aligned where $\theta = \|\theta_0\|_2$. We consider parameters $\theta^\star \in \{ (\theta, \pm \bar \Delta, ... , \pm \bar \Delta) \in \R^{d+1}\} $. Since the action space is the unit ball, the our optimal action is $$a^\star = \frac{\theta^\star}{ \|\theta^\star\|} = \frac{\theta^\star}{\sqrt{\theta^2 + \Delta^2}} = \frac{\theta^\star}{r},$$
where $r = \sqrt{\theta^2 + \Delta^2}$ is the norm of all considered optimal parameters.
The expected regret after $T$ rounds in bandit instance $\theta^\star$ is
\begin{align}
    \EE[\regret_{\theta^\star}(T)] &= \sum_{t=1}^T \EE\left[ \langle a^\star - a_t, \theta^\star \rangle \right] \nonumber\\
    &= 
    \sum_{t=1}^T \theta\left(\frac{\theta}{r} - \EE [a_{t1}] \right) + \sum_{t=1}^T \sum_{i=2}^{d+1} \bar \Delta \EE\left( \frac{\bar \Delta}{r} - a_{ti} \cdot \sign(\theta^\star_i)\right) \nonumber\\
    &= \sum_{t=1}^n \theta\left(\frac{\theta}{r} - \EE [a_{t1}] \right) + \sum_{t=1}^T \sum_{i=2}^{d+1} \bar \Delta \EE\left( \frac{\bar \Delta}{r} - b_{ti} \right) 
    \label{eqn:regret_form}
\end{align}
where $b_{t1} = a_{t1}$ and $b_{ti} = a_{ti} \cdot \sign(\theta^\star_i)$ for $2 \leq i \leq d+1$.  Notice that 
\begin{align*}
\frac{r}{2\bar \Delta}\sum_{i=2}^{d+1} \left(\frac{\bar \Delta}{r} - b_{ti}\right)^2 &= \sum_{i=2}^{d+1} \left[\frac{\bar \Delta}{2r} - b_{ti} + \frac{r}{2\bar \Delta} b_{ti}^2 \right] 
\overset{(i)}{\leq} -\frac{\bar \Delta d}{2r} +  \frac{r}{2\bar \Delta} (1 - b_{t1}^2) + \sum_{i=2}^{d+1} \left[\frac{\bar \Delta}{r} - b_{ti}\right],
\end{align*}
where step $(i)$ follows from $1 \geq \|b_t\|^2 = b_{t1}^2 + \sum_{i=2}^{d+1} b_{ti}^2$. Now, denote $X_t := \frac{\theta}{r} -  b_{t1}$ and $Y_t := -\frac{\bar \Delta d}{2r} +  \frac{r}{2\bar \Delta} (1 - b_{t1}^2)$. Then

\begin{align*}
    2Y_t &= \frac{\bar \Delta d}{r} + \frac{r}{\bar \Delta}\left(1 - \left(\frac{\theta}{r} - X_i\right)^2 \right) 
    = \frac{\bar \Delta d}{r} + \frac{r}{\bar \Delta} \left(1 - \frac{\theta^2}{r^2} + 2 \frac{\theta}{r}X_t - X_t^2 \right)\\
    &= \underset{=0}{\underbrace{\frac{\bar \Delta d}{r} + \frac{r}{\bar \Delta} \left(1 - \frac{\theta^2}{r^2}\right)}} + 2 \frac{\theta}{\bar \Delta}X_t - \frac{r}{\bar \Delta}X_t^2 
    = 2 \frac{\theta}{\bar \Delta} X_t - \frac{r}{\bar \Delta}X_t^2
\end{align*}
and thus
$Y_t = \frac{\theta}{\bar \Delta} X_t - \frac{r}{2\bar \Delta}X_t^2$. Using these identities, we rewrite the regret in Eq.~\eqref{eqn:regret_form} as
\begin{align*}
    \EE[\regret_{\theta^\star}(T)]
    &\geq \sum_{t=1}^T \theta\EE[X_t]  + \sum_{t=1}^T  \bar \Delta \EE\left( \sum_{i=2}^{d+1}\frac{r}{2\bar \Delta}\left[ \frac{\bar \Delta}{r} - b_{ti} \right]^2 - Y_t\right)\\
    &= \sum_{t=1}^T \theta\EE[X_i]  + \sum_{t=1}^T \EE \left( \frac{r}{2} \sum_{i=2}^{d+1}\left[ \frac{\bar \Delta}{r} - b_{ti} \right]^2 - \theta X_t + \frac{r}{2} X_t^2\right)\\
    &= \frac{r}{2} \sum_{t=1}^T\EE \left[X_t^2\right] + \sum_{i=2}^{d+1} \frac{r}{2}  \EE\left[\sum_{t=1}^T \left(\frac{\bar \Delta}{r} - b_{ti} \right)^2\right]\\
    & \geq \sum_{i=2}^{d+1} \frac{r}{2}  \EE\left[\sum_{t=1}^T \left(\frac{\bar \Delta}{r} - b_{ti} \right)^2\right]~.
\end{align*}
We now follow the analysis of Theorem~24.2 by \citet{lattimore2020bandit}.
Define $\tau_i = T \wedge \min\{t \colon \sum_{s=1}^t a_{si}^2 \geq T\bar \Delta^2 /  r^2\}$, which is a stopping time and $U_i(x) = \sum_{t=1}^{\tau_i} \left(\frac{\bar \Delta}{r} - a_{ti} x\right)^2$ for $x \in \{-1, +1\}$. Then we can lower-bound the expression above as
\begin{align*}
    \sum_{i=2}^{d+1} \frac{r}{2}  \EE\left[\sum_{t=1}^T \left(\frac{\bar \Delta}{r} - b_{ti} \right)^2\right]
    \geq \sum_{i=2}^{d+1} \frac{r}{2}  \EE\left[\sum_{t=1}^{\tau_i} \left(\frac{\bar \Delta}{r} - b_{ti} \right)^2\right]
    = \sum_{i=2}^{d+1} \frac{r}{2}  \EE\left[U_i(\sign \theta^\star_i)\right].
\end{align*}
Let $\theta'$ an alternative to $\theta^\star$. By Pinsker's inequality, we have
\begin{align*}
    \EE_{\theta^\star}[U_i(1)] &\geq \EE_{\theta'}[U_i(1)] - \sqrt{\frac{1}{2}D(\PP^\star, \PP')} \sup U_i(1)
    \geq 
\EE_{\theta'}[U_i(1)] - \left( 4 T \frac{\bar \Delta^2}{r^2} + 2 \right)\sqrt{\frac{1}{2}D(\PP^\star, \PP')} 
\end{align*}
where the second inequality follows from the following bound on $U_i(1)$ 
\begin{align*}
    U_i(1) = \sum_{t=1}^{\tau_i} \left(\frac{\bar \Delta}{r} - a_{ti} \right)^2
    \leq 2 \tau_i \frac{\bar \Delta^2}{r^2} + 2 \sum_{t=1}^{\tau_i}  a_{ti}^2
    \leq 4 T \frac{\bar \Delta^2}{r^2} + 2.
\end{align*}
We bound $D(\PP^\star, \PP') \leq \EE_{\theta^\star}[T_i(\tau_i)] D(\PP^\star_i, \PP'_i)
\leq \frac{\bar \Delta^2}{2} \EE_{\theta^\star}\left[\sum_{t=1}^{\tau_i} a_{ti}^2\right] 
\leq \frac{\bar \Delta^2}{2} \left(\frac{T \bar \Delta^2}{r^2} + 1\right)$. 
Then we have
\begin{align*}
    \EE_{\theta^\star}[U_i(1)] + \EE_{\theta'}[U_i(-1)] &\geq \EE_{\theta'}[U_i(1) + U_i(-1)] - \frac{\bar \Delta}{2}
    \left( 4 T \frac{\bar \Delta^2}{r^2} + 2 \right)\sqrt{\frac{T\bar \Delta^2}{r^2} + 1}\\
    &= 2\EE_{\theta'}\left[ \frac{\tau_i\bar \Delta^2}{r^2} + \sum_{t=1}^{\tau_i} a_{ti}^2\right] - \frac{\bar \Delta}{2}
    \left( 4 T \frac{\bar \Delta^2}{r^2} + 2 \right)\sqrt{\frac{T\bar \Delta^2}{r^2} + 1}
    \\
   & \geq 2T \frac{\bar \Delta^2}{r^2} - \frac{\bar \Delta}{2}
    \left( 4 T \frac{\bar \Delta^2}{r^2} + 2 \right)\sqrt{\frac{T\bar \Delta^2}{r^2} + 1} =: G.
\end{align*}
We denote by 
 $G = 2T \frac{\bar \Delta^2}{r^2} - \frac{\bar \Delta}{2}
    \left( 4 T \frac{\bar \Delta^2}{r^2} + 2 \right)\sqrt{\frac{T\bar \Delta^2}{r^2} + 1}$. Then by the randomization hammer:
\begin{align*}
    \sum_{\theta' \in \{\pm \bar \Delta\}^d} \EE[\regret_{\theta'}(T)]
    &\geq \frac{r}{2} \sum_{\theta' \in \{\pm \bar \Delta\}^d} 
    \sum_{i=2}^{d+1}   \EE_{\theta'}\left[ U_i(\sign \theta_i')\right]\\
    &= \frac{r}{2} \sum_{i=2}^{d+1} \sum_{\theta'_{-i} \in \{\pm \bar \Delta\}^{d-1}} \sum_{\theta'_i \in \{\pm \bar \Delta\}} 
    \EE_{\theta'}\left[ U_i(\sign \theta_i')\right]\\
    & \geq \frac{r}{2} \sum_{i=2}^{d+1} \sum_{\theta'_{-i} \in \{\pm \bar \Delta\}^{d-1}} G = 2^d \frac{G rd}{4 }.
\end{align*}
Since the average regret over all $\theta'$ is at least $G r d / 4$, there is at least one $\theta^\star$ such that
\begin{align}
    \EE[\regret_{\theta^\star}(T)] \geq \frac{Grd}{4} &= \frac{Td\bar \Delta^2}{2r} - \frac{r d \bar \Delta}{8}
    \left( 4 T \frac{\bar \Delta^2}{r^2} + 2 \right)\sqrt{\frac{T\bar \Delta^2}{r^2} + 1}\\
    &= 
    \frac{T \Delta^2}{2r} - \frac{r \sqrt{d} \Delta}{8}
    \left( 4 T \frac{\Delta^2}{dr^2} + 2 \right)\sqrt{\frac{T \Delta^2}{d r^2} + 1}
\\
    &\geq 
        \frac{T \Delta^2}{2r} - \frac{r \sqrt{d} \Delta}{4}
    \left( 2 T \frac{\Delta^2}{dr^2} + 1 \right)^{3/2}~.
\end{align}
Assume now that $0 \leq \Delta \leq \sqrt{\frac{3}{4}}$. Then $r \in [0.25, 1]$ and the lower-bound above simplifies to
\begin{align}
    \EE[\regret_{\theta^\star}(T)] &\geq 
        \frac{T \Delta^2}{2} - \frac{\sqrt{d} \Delta}{4}
    \left( \frac{32}{d} T \Delta^2 + 1 \right)^{3/2}.
\end{align}
\end{proof}

\thmotherlower*

\begin{proof}
Notice that our assumptions on $\Delta$, $T$ and $h$ imply $\|h/2\|=1/2$ and $\Delta \le \frac{1}{4}\le \sqrt{\frac{3}{4}}$. Thus,
by \pref{lem:noncenteredlowerbound}, there is some $\theta^\star$ satisfying $\|\theta^\star-\frac{h}{2}\|\le \Delta$ and $\langle \theta^\star,a^\star -h\rangle\le \Delta^2$ such that 
\begin{align*}
    \EE[\regret_{\theta^\star}(T)]&\geq\frac{T \Delta^2}{2} - \frac{\sqrt{d} \Delta}{4}
    \left( \frac{32}{d} T \Delta^2 + 1 \right)^{3/2}\\
    &\ge \frac{T \Delta^2}{2} - \frac{\sqrt{d} \Delta}{4}
    \left( \frac{\sqrt{T}}{8} + 1 \right)^{3/2}
    \intertext{use $T\ge 2^{6}$:}
    &\ge \frac{d\sqrt{T}}{256} - \frac{d}{64 T^{1/4}}
    \left( \frac{\sqrt{T}}{4} \right)^{3/2}\\
    & = \frac{d\sqrt{T}}{512}
\end{align*}

Furthermore, notice that $1/2-\Delta=\|h/2\|-\Delta \le \|\theta^\star\|\le \|h/2\|+\Delta=1/2+\Delta$. Now, since $f(x) = 1/x$ is convex for positive $x$, we have:
\begin{align*}
    \frac{1}{1/2+\Delta}\ge 2-\Delta
\end{align*}
Further, by our conditions on $T$, $\Delta \le 1/4$. Thus $|f'(x)|\le 2$ for $x\in[1/2-\Delta, 1/2+\Delta]$ so that:
\begin{align*}
    \frac{1}{1/2-\Delta}\le 2-2\Delta 
\end{align*}
Putting all this together:
\begin{align*}
    \left\|\frac{\theta^\star}{\|\theta^\star\|} - h\right\|&\le \|h-2\theta^\star\|+\|\theta^\star\|\left|2-\frac{1}{\|\theta^\star\|}\right|\\
    &\le 4\Delta
\end{align*}

Finally, we need to bound $\langle \theta^\star , a^\star -h\rangle$. To this end, observe that $\|h-2\theta^\star\|\le 2\Delta$. Then, since $\Delta \le 1/4$ and $\|h\|=1$, we must have $\langle \theta^\star, h\rangle \ge 2\|\theta^\star\|^2-2\Delta\|\theta^\star\|\ge -\|\theta^\star\|/2$. Therefore, by \pref{lem:regret_unitball}, we have
\begin{align*}
    \langle \theta^\star, a^\star - h\rangle &\le 3 \frac{\|P_h^\perp \theta^\star\|^2}{\|\theta^\star\|}
\end{align*}
Now, let us define $\epsilon = h-2\theta^\star$. Notice that $\|\epsilon\|\le 2\Delta\le 1/2$. Observe that 
\begin{align*}
    P_h^\perp \theta^\star &= \theta^\star - \frac{\langle h,\theta^\star\rangle h}{\|h\|^2}\\
    &=\theta^\star - 2\langle h,\theta^\star\rangle\theta^\star - \langle h,\theta^\star\rangle \epsilon\\
    &=\theta^\star - 4\|\theta^\star\|^2\theta^\star -(\langle h,\theta^\star\rangle+2\|\theta^\star\|^2)\epsilon 
\end{align*}
Now, we have $(1/2-\Delta)^2\le \|\theta^\star\|^2\le (1/2+\Delta)^2$. Since $\Delta\le 1/4$, this yields $1/4-\Delta \le \|\theta^\star\|\le 1/4 + \frac{3}{2}\Delta$ so that
\begin{align*}
    \|P_h^\perp \theta^\star\|\le 9\Delta
\end{align*}
Thus overall we obtain $\langle \theta^\star, a^\star -h\rangle \le 972\Delta^2$
\end{proof}

\section{Proof of Regret Upper Bounds}
\actionregret*

\begin{proof}
Note that $P_{\theta^*}a = \frac{\langle a,\theta^*\rangle \theta^*}{\|\theta^*\|^2}$.
\begin{align*}
 \langle a^* - a , \theta^* \rangle 
 &=  
 \left\langle \frac{\theta^*}{\|\theta^*\|} - a , \theta^* \right\rangle = \left\langle \frac{\theta^*}{\|\theta^*\|} - \frac{\langle a,\theta^*\rangle \theta^*}{\|\theta^*\|^2} , \theta^* \right\rangle 
 = \|\theta^*\|\left\langle \frac{\theta^*}{\|\theta^*\|^2} - \frac{\langle a,\theta^*\rangle \theta^*}{\|\theta^*\|^3}  , \theta^* \right\rangle \\
&= \|\theta^*\|\left(1 - \frac{\langle a,\theta^*\rangle}{\|\theta^*\|} \right)\\
\intertext{Now, use $\frac{\langle a,\theta^*\rangle}{\|\theta^*\|}\in[-1/2,1]$ and the observation $1-x\le 3(1-|x|)$ for $x\in[ -1/2, 1]$:}
&\le 3\|\theta^*\| \left(1 - \frac{|\langle a,\theta^*\rangle|}{\|\theta^*\|} \right)
= 3\|\theta^*\|\left(1 - \|P_{\theta^*}a\| \right)
= 3\|\theta^*\|\left(1 - \sqrt{1 - \|P_{\theta^*}^\perp a\|^2} \right)
\intertext{now, since $ 1-\sqrt{1-x}\le x$ for $x\in [0,1]$}
&\le 3\|\theta^*\| \|P_{\theta^*}^\perp a\|^2 = 3 \frac{\|P^\perp_a \theta^*\|}{\|\theta^*\|}~.
\end{align*}
Conversely, since $1 - x \geq 1 - |x|$ and $\frac{x}{2} \leq 1 - \sqrt{1 - x}$, we also have
\begin{align*}
    \langle a^* - a , \theta^* \rangle = \|\theta^*\|\left(1 - \frac{\langle a,\theta^*\rangle}{\|\theta^*\|}\right) 
    \geq \|\theta^*\| \left(1 - \frac{|\langle a,\theta^*\rangle|}{\|\theta^*\|} \right)
    \geq \frac{1}{2} \|\theta^*\| \left(\|P_{\theta^*}^\perp a\|^2 \right)~.
\end{align*}
Finally use the identity $\|P^\perp_{\theta^\star} a\| \|\theta_\star\| = \|P^\perp_{a} \theta^\star\| \|a\|$.
\end{proof}

\begin{proof}[Proof of \pref{lem:prior_regret_perturb}]
We can write the regret of $a$ with respect to $h$ as
\begin{align*}
\langle h - a, \theta^\star \rangle &= \langle h - \frac{h + p}{\sqrt{1 + \|p\|^2}}, \theta^\star \rangle = (1 - \frac{1}{\sqrt{1+\|p\|^2}})\langle h, \theta^\star \rangle + \frac{1}{\sqrt{1+\|p\|^2}} \langle p, \theta^\star \rangle \\
 \intertext{Since $\langle h, \theta^\star\rangle \leq \|h\|\|\theta^\star\| = \|\theta^\star\|$ and for any $y \in (0, 1)$ we have $1 - \frac{1}{\sqrt{1 + y}} \leq y \Leftrightarrow (1 - y) \sqrt{1 + y} \leq 1 \Leftrightarrow (1 - y^2) (1 - y) \leq 1$, we can upper-bound this by}
 &\leq \|p\|^2\|\theta^\star\| +  |\langle p, \theta^\star\rangle| \leq \|p\|^2\|\theta^\star\| + |\langle p, P_h^{\perp} \theta^\star \rangle|
\end{align*}
where the final inequality holds because $\langle p, h \rangle = 0$. 

Furthermore, if $p \sim \N(0, \frac{1}{d} {\bf I})$, then we see that with probability at least $1-\delta$, $\|p\|^2 \leq O(\log(1/\delta))$ and since $\langle p, P_h^\perp a^\star \rangle \sim N(0, \frac{\|P_h^\perp a^\star\|^2}{d})$, we have that with probability at least $1-\delta$, 

$$|\langle p, P_h^\perp a^\star\rangle | \leq \frac{\|P_h^\perp a^\star\|\sqrt{\log(1/\delta)}}{\sqrt{d}} $$

Therefore, we conclude that with probability at least $1-\delta$, our instantaneous regret bound follows by scaling down $p$ by $\Delta$ and applying the same argument in the projected $d-1$ dimension subspace given by the projection $P_h^\perp = {\bf I} - h h^\top$. Note that we may rewrite $\|P_h^\perp a^\star\| = \|P_h^\perp \theta^\star\|/\|\theta^\star\|$ to get our final theorem.

\end{proof}

\begin{proof}[Proof of \pref{lem:regret_perturb}]
Since $\langle h , \frac{ \theta^\star}{\|\theta^\star\|} \rangle  \geq -1/4$ and $\| p\| < 1/8$, we have $\langle a , \frac{\theta^\star}{\|\theta^\star\|} \rangle \geq - 1/2$. Therefore, we can apply \pref{lem:regret_unitball} and bound the instantaneous regret of $a$  as
\begin{align*}
    \langle a^\star - a , \theta^\star \rangle
    \leq 3\|\theta^\star\| \|P_{\theta^\star}^\perp a\|^2 = 3\|\theta^\star\|(1 - \|P_{\theta^\star} a\|^2).
\end{align*} 
To complete the proof, it suffices to show that 
$\| P_{\theta^\star} a \|^2 \geq 1 - 4 \frac{\|P_h^\perp \theta^\star \|^2}{\|\theta^\star\|} - \|p\|^2$
which we do in the following.
To simplify notation let $\theta^\star = \alpha h + v$ where $\alpha \in \R$ and $v = P_h^\perp \theta^\star \in \R^d$ with $\langle v, h \rangle = 0$. Then bound
\begin{align*}
    \|P_{\theta^\star} a\|^2 &= \left\langle \frac{\theta^\star}{\|\theta^\star\|}, a \right\rangle^2 = \frac{1}{\alpha^2 + \|v\|^2} \langle \alpha h + v, a\rangle^2 \\
    &\geq \frac{\alpha^2}{\alpha^2 + \|v\|^2} \langle h, a\rangle^2  - \frac{2}{\alpha^2 + \|v\|^2 } |\alpha   \langle v, a\rangle| +  \frac{1}{\alpha^2 + \|v\|^2}\langle v, a\rangle^2 \\
    &\geq  (1-\frac{\|v\|^2}{\alpha^2 + \|v\|^2})\langle h, a\rangle^2
    -  \frac{2 |\alpha|}{\alpha^2 + \|v\|^2 }  | \langle v,  p\rangle| \\
    &\geq  \left(1 - \frac{\|v\|^2}{\|\theta^\star\|^2}\right) \frac{1}{1+\| p\|^2} - \frac{2 |\alpha| \| p\|\|v\|}{\|\theta^\star\|^2 } 
    \intertext{using $\frac{1}{1+x}\le 1-x/2$ for $x\in[0,1]$:}
    &\geq  1 - \frac{\|v\|^2}{\|\theta^\star\|^2} - \frac{\|p\|^2}{2} -  \frac{2|\alpha| \| p\|\|v\|}{\|\theta^\star\|^2 } \\
    &\geq  1 - \frac{\|v\|^2}{\|\theta^\star\|^2} - \frac{\| p\|^2}{2} -  \frac{2\| p\|\|v\|}{\|\theta^\star\|}
    \intertext{applying young inequality $2xy\le \frac{x^2}{\lambda} + \lambda y^2$ with $\lambda=2$}
    &\geq  1 - 4\frac{\|v\|^2}{\|\theta^\star\|^2} - \| p\|^2
\end{align*}
\end{proof}

\estimationproc*
\begin{proof} We first show correctness and then bound the number of calls before the a value is returned.
\paragraph{Correctness}
We first compute the expectation of the sample averages $\bar y_n$, $\bar z_n$ and $\bar x_n$ for all $n$, assuming $p$ to be fixed,
\begin{align*}
    \E[ \bar y_n] &= \E[ y_i] = \langle \theta^\star, h\rangle,\\
    \E[ \bar z_n] &= \E[ z_i] \left\langle \theta^\star, \frac{h + p}{\sqrt{\|h\|^2 + \| p\|^2}}\right\rangle
    =   \frac{\langle \theta^\star,h \rangle  + \langle \theta^\star,p \rangle}{\sqrt{\|h\|^2 + \| p\|^2}},
    \\
    \E[ \bar x_n] &= \E[ \bar z_n]\sqrt{\|h\|^2 + \| p\|^2} - \E[ \bar y_n]
    = \langle \theta^\star,p \rangle~.
\end{align*}
We can write $\bar x_n$ as an average of $x_i = \sqrt{\|h\|^2 + \|p\|^2} z_i - y_i$. Since $z_i$ and $y_i$ are each $1$-sub-Gaussian random variables, $x_i$ is $\sqrt{1 + \|h\|^2 + \|p\|^2}$-sub-Gaussian (each after being centered). We can now apply  an anytime-version of the standard Hoeffding concentration argument (see \pref{lem:time_uniform_hoeffding}) to get that with probability at least $0.9$ for all $n \in \mathbb N$
\begin{align}\label{eqn:xconc}
    |\langle \theta^\star,p \rangle - \bar x_n| 
    \leq \sqrt{\frac{3(1 + \|h\|^2 + \|p\|^2)\ln(40 \ln(2n))}{n}} =: b_n ~.
\end{align}
The algorithm returns a value if and only if the magnitude of the empirical average is at least twice the confidence width, i.e., $|\bar x_n| \geq 2 b_n$. Using this condition and \pref{eqn:xconc}, we have
\begin{align}\label{eqn:factor_approx1}
   \frac{| \bar x_n|}{2} =  |\bar x_n| - \frac{1}{2} | \bar x_n|
   \leq | \bar x _n| - b_n  \leq 
    |\langle \theta^\star,p \rangle| &\leq | \bar x _n| + b_n \leq |\bar x_n| + \frac{1}{2} | \bar x_n| = \frac{3| \bar x_n|}{2}~,
\end{align}
that is, $\|\bar x_n\|$ is a constant factor approximation of $|\langle \theta^\star, p\rangle|$. We now argue using the distribution of $p$ that  $|\langle \theta^\star, p\rangle|$ is a constant factor approximation of $\|P^\perp_h \theta^\star\|\frac{\Delta}{\sqrt{d'}}$.
This holds because $\frac{1}{\Delta} p$ is a $d'$-dimensional isotropic Gaussian random variable (in $\mathbb R^d$ when $h = 0$ and the $d-1$-dimensional orthogonal complement of $h$ when $h \neq 0$). We can therefore show that $\frac{\langle \theta^\star, p \rangle^2}{\Delta^2 \|P^\perp_h \theta^\star\|^2}$ follows a $\chi^2_{1}$ for which we can bound its tail probabilities (see \pref{lem:gaussian_norm_bounds}) to get that 
\begin{align}\label{eqn:factor_approx2}
    0.1 \frac{\Delta \|P_h^\perp \theta^\star\|}{\sqrt{d'}} \leq |\langle \theta^\star, p \rangle| \leq 2.33 \frac{\Delta \|P_h^\perp \theta^\star\|}{\sqrt{d'}}
\end{align}
holds with probability at least $0.9$.
Combining \pref{eqn:factor_approx1} and \pref{eqn:factor_approx2}, we get the desired constant-factor approximation of the return value
\begin{align}
    |x_n|\frac{\sqrt{d-1}}{\Delta} \in \left[0.06 \|P^\perp_h \theta^\star\|, 5 \|P^\perp_h \theta^\star\|\right]~.
\end{align}
Note that this result holds with probability at least $0.8$, by taking a union bound over the events of 
\pref{eqn:xconc} and \pref{eqn:factor_approx2}. 
\paragraph{Number of rounds.}
We now bound the number of rounds $n$ until the return condition $|\bar x_n| \geq 2 b_n$ is satisfied in the events considered above.
If this condition is violated in round $n$, i.e., $|\bar x_n|  < 2  b_n$, then by rearranging this inequality, we have $|\langle \theta^\star, p\rangle| \leq  | \bar x_n| + b_n < 2 b_n + b_n =3 b_n$. Combining this with \pref{eqn:factor_approx2} gives that if the algorithm does not return a value in round $n$, then 
\begin{align*}
  \frac{0.1 \Delta}{ \sqrt{d'}} \|P^\perp_h \theta^\star\| < 3 b_n =  3\sqrt{\frac{3(1 + \|h\|^2 + \|p\|^2)\ln(40 \ln(2n))}{n}}
\end{align*}
and thus $n = O\left( \frac{d(1 + \|p\|^2 + \|h \|^2)}{\Delta^2 \| P_h^\perp \theta^\star\|^2}
\ln \ln \frac{d(1 + \|p\|^2 + \|h \|^2)}{\Delta^2 \| P_h^\perp \theta^\star\|^2}
\right) =
\tilde O \left( \frac{d(1 + \|p\|^2 + \|h \|^2)}{\Delta^2 \| P_h^\perp \theta^\star\|^2}
\right)$~.
Finally, by \pref{lem:gaussian_norm_bounds}, we have $\|p\| \leq 3 \Delta$ with probability at least $0.9$ and thus $
n = 
\tilde O \left( \frac{d(1 + \Delta^2 + \|h \|^2)}{\Delta^2 \| P_h^\perp \theta^\star\|^2}
\right)$.
\end{proof}

\regretlowprob*

\begin{proof}
Since $\|p\| \leq 3\Delta \leq 1/8$, we can apply \pref{lem:regret_perturb} and bound this regret as
\begin{align*}
    n \left( 24\frac{\|P_h^\perp \theta^\star\|^2}{\|\theta^\star\|} + 3\|\theta^\star\| \Delta^2\right)
    = \tilde O \left( 
    \frac{d}{\Delta^2 \|  \theta^\star\|}
    + 
    \frac{d \| \theta^\star\|}{\| P_h^\perp \theta^\star\|^2}
    \right)~.
\end{align*}
The regret with respect to the reference action $h$ is bounded as by \pref{lem:prior_regret_perturb} as
\begin{align*}
   n ( \langle h - a, \theta^\star \rangle) &\leq n\|\theta^\star\|\| p\|^2 + n\|\theta^\star\||\langle p, a^\star\rangle|\\
   &\leq 9 n \|\theta^\star\| \Delta^2 + 2.33 n \frac{\Delta \|P_h^\perp \theta^\star\|}{\sqrt{d'}}\\
   & = \tilde O \left(  
    \frac{\sqrt{d}}{\Delta \| P^\perp_h \theta^\star\|}
    + 
    \frac{d \| \theta^\star\|}{\| P_h^\perp \theta^\star\|^2} \right)
\end{align*}
because $\|\theta^\star\||\langle p, a^\star\rangle| = |\langle p, \theta^\star\rangle| \leq 2.33 \frac{\Delta \|P_h^\perp \theta^\star\|}{\sqrt{d'}}$
\end{proof}

\begin{algorithm2e}
\SetAlgoVlined
\SetKwInOut{Input}{Input}
\SetKwProg{myproc}{Procedure}{}{}

\DontPrintSemicolon
\LinesNumbered
\Input{reference action $h \in \R^d $, perturbation magnitude  $\Delta \in \R^+$, failure probability $\delta$}
Set $k = 560\ln(1/\delta)$, Initialize active set $\mathcal S = [k]$ and return set $\mathcal R = \varnothing$\;
Initialize $k$ instances of \pref{alg:estimation} as $C_i = \textsc{EstimateNorm}(h, \Delta)$ for $i \in [k]$\;
\myproc{\textsc{PlayAndUpdate}$()$}
 {
 \If{$| \mathcal R| \geq 0.67 k$}{
 play hint $h$, observe reward $y_n$.
}\Else{
 \For{$i \in \mathcal S$}{
    Call $r_i = C_i.\textsc{PlayAndUpdate()}$\;
    \If{$r_i$ is not none}{
    $\mathcal S \gets \mathcal S \setminus \{ C_i\}$ and $\mathcal R \gets \mathcal{R} \cup \{r_i\}$\;
    }
 }}
 \If{$| \mathcal R| \geq 0.67 k$}{
 \Return median$(\mathcal R)$\;
 }
    
}

\caption{\textsc{EstimateNormHP}$(h,\Delta, \delta)$: High Probability Low Regret $\ell_2$-norm Estimation}
\label{alg:hp_estimation2e}
\end{algorithm2e}

\estimatenormhp*

\begin{proof}
Denote by $F_i$ the event where the $i$th instance of \pref{alg:estimation} fails, i.e., where the statement in \pref{lem:estimation_procedure} does not hold. All failure events are independent from each other and have probability at most $\PP(F_i) \leq 0.3$ by \pref{lem:estimation_procedure}. We here consider the event $E$ where at least $0.67k$ instances succeed. The probability of this event is at least
\begin{align*}
  1 -  \PP\left(\sum_{i=1}^k \indicator{F_i^c} < 0.67k\right)
  \geq 1 - \exp\left(-2k\left( 0.7 - \frac{0.67k}{k}\right)^2\right) \geq  1 - \delta
\end{align*}
by Hoeffding's inequality. We know that in $E$, \pref{alg:hp_estimation2e} returns a value $r$ after at most  $\tilde O \left( \frac{d}{\Delta^2 \| P_h^\perp \theta^\star\|^2}
\right)$ calls to its \textsc{PlayAndUpdate} procedure. 
Further, since at most $0.33k$ instances fail in $E$, the majority of entries in the return set $\mathcal R$ was generated by a succeeding instance. Hence, the constant approximation guarantee of \pref{alg:estimation} also holds for \pref{alg:hp_estimation2e}.

Next, we can apply \pref{lem:prior_regret_perturb} to bound the regret with respect to $h$ as:
\begin{align*}
   n ( \langle h - a, \theta^\star \rangle) &\leq n\|\theta^\star\|\| p\|^2 + n\|\theta^\star\||\langle p, P^\perp_h a^\star\rangle|\\
   &\le 9 n \|\theta^\star\| \Delta^2  + n|\langle p, \theta^\star\rangle|\\
   &\leq 9 n \|\theta^\star\| \Delta^2 + 2.33 n \frac{\Delta \|P_h^\perp \theta^\star\|}{\sqrt{d'}}\\
   & = \tilde O \left(  \frac{d \| \theta^\star\|}{\| P_h^\perp \theta^\star\|^2}\ln\frac{1}{\delta}+
    \frac{\sqrt{d}}{\Delta \| P^\perp_h \theta^\star\|}\ln\frac{1}{\delta}
     \right)
\end{align*}
because $\|\theta^\star\||\langle p, a^\star\rangle| = |\langle p, \theta^\star\rangle| \leq 2.33 \frac{\Delta \|P_h^\perp \theta^\star\|}{\sqrt{d'}}$ by \pref{eqn:factor_approx2}.

Finally, the total number of calls to \textsc{PlayAndUpdate} of all $C_i$ instances (failing and succeeding) is
$n = \tilde O \left( \frac{d}{\Delta^2 \| P_h^\perp \theta^\star\|^2} \ln \frac{1}{\delta}
\right)$, and since $\|p\|\le 3\Delta\le 1/8$, each of the $2n$ samples collected satisfies the conditions of 
\pref{lem:regret_perturb}. We can therefore bound the total regret as
\begin{align*}
    2n \left( 24\frac{\|P_p^\perp \theta^\star\|^2}{\|\theta^\star\|} + 3\|\theta^\star\| \Delta^2\right)
    = \tilde O \left( 
    \frac{d(1 + \|h \|)}{\Delta^2 \|  \theta^\star\|} \ln \frac{1}{\delta}
    + 
    \frac{d(1 + \|h \|)  \| \theta^\star\|}{\| P_h^\perp \theta^\star\|^2}\frac{1}{\delta}
    \right)~.
\end{align*}

\end{proof}

\begin{proof}[Proof of \pref{thm:main_single_prior_2e}]
First, observe that if $\|\theta^\star\|\le \max(d,3374)/\sqrt{T}$, then any sequence of actions would obtain regret $O(d\sqrt{T})$. Thus, the interesting regime is $\|\theta^\star\| \geq \max(d,3374)/\sqrt{T}$, which we consider for the remainder of the proof.

We split our regret analysis into 3 phases, as labeled in the psuedocode: in the first phase we estimate $\|\theta^\star\|$, in the second phase we estimate $\|P^\perp_h \theta^\star\|$, and finally in the last phase we call $\textsc{Switch}$. Note that our algorithm may not execute all three segments before we reach our iteration budget $T$.

Notice that if $\langle \theta^\star, h\rangle \le -\|\theta^\star\|/4$, then to show that $R^T_h\le \tilde O(\sqrt{T})$, it suffices to instead show that $R^T_{-h}\le \tilde O(\sqrt{T})$. Thus by possibly swapping $-h$ and $h$, we may assume $\langle \theta^\star, h\rangle \ge -\|\theta^\star\|/4$.

Let $E_0$, $E_+$ and $E_-$ be the events such that the conclusion of \pref{lem:estimatenormhp2e} holds for $C_0$, $C_+$ and $C_-$ respectively. Notice that by \pref{lem:estimatenormhp2e}, each of these events has probability at least $1-\delta/4$. Further, let $E_{Y,+}$ and $E_{Y,-}$ be the respective events that $\langle \theta^\star, h\rangle\in Y_+$ and $\langle \theta^\star,-h\rangle\in Y_-$ for all $n$. By \pref{lem:time_uniform_hoeffding}, we have that $E_{Y,+}$ and $E_{Y,-}$ each occur with probability at least $1-\delta/10$. Let $E_{LB}$ be the event that the linear bandit algorithm used by \textsc{Switch} has regret at most $W_3 d\sqrt{t}$ for all $t\le T$. There exists an absolute constant $W_3$ such that $E_{LB}$ occurs with probability at least $1-\delta/20$. Let $E$ be the union of all these events. Clearly $E$ has probability at least $1-\delta$. We condition the rest of our argument on this event.

\textbf{Phase 1:} For the first phase, since $\|\theta^\star\| \geq d/ \sqrt{T}$, by \pref{lem:estimatenormhp2e}, we obtain $r$ satisfying $0.06\|\theta^\star\|\le r\le 5\|\theta^\star\|$ after $\tilde O(d\ln(1/\delta)/\|\theta^\star\|^2)$ rounds, incurring $\tilde O\left(\frac{d}{\|\theta^\star\|}\ln\frac{1}{\delta}\right) = \tilde O(\sqrt{T}\ln(1/\delta))$ regret. Since the total regret is bounded by $\tilde O(\sqrt{T}\ln(1/\delta))$, the hint-based regret is also similarly bounded. Further, since $\|\theta^\star\|\ge 3374/\sqrt{T}$, this implies $\Delta \le \frac{1}{224}$.

\textbf{Phase 2:} For the second phase, we call Algorithm~\ref{alg:hp_estimation2e} on $h$ and $-h$ as instances $C_+$ and $C_-$. We need to verify three facts: first, the value $r_\perp$ produced by this phase needs to be a constant-factor approximation to  $\|P^\perp_h\theta^\star\|$. Second, the hint-based regret incurred during this phase must be at most $\tilde O(\sqrt{T})$. Finally, the worst-case regret incurred during this phase must be at most $\tilde O(d\sqrt{T})$. 

We will consider two broad cases: either we exit phase 2 with at most one of $C_+$, $C_-$ having returned, or not. Let us first consider the case that at most one instance returns. Notice that since $\langle \theta^\star,h\rangle \ge \langle \theta^\star -h\rangle$, under our assumed events $E_{Y,\pm}$, we will never eliminate $C_+$ in line 12 of the algorithm. Thus, the only way this case can occur is if the first algorithm to return immediately triggers one of the conditions on line 13 and exits Phase 2 (or no instances return).


Notice that $P^\perp_h=P^\perp_{-h}$. Thus by \pref{lem:estimatenormhp2e}, the return values of both $C_+$ and $C_-$ will satisfy $0.06\|P^\perp_h\theta^\star\|\le r_\perp\le 5\|P^\perp_h\theta^\star\|$, so regardless of which subroutine provides $r_\perp$, it will be a constant-factor approximation to $\|P^\perp_h\theta^\star\|$. 

Now, it remains to bound the regret. Let $N+1$ be the number of times $C_-.\textsc{PlayAndUpdate}()$ is called, so that $Y_+\cap Y_i\ne \varnothing$ after $N$ calls. Notice that $\langle \theta^\star, h\rangle\in Y_{+}$ and $\langle \theta^\star, -h\rangle \in Y_-$, and after $N$ calls, $|Y_+|=|Y_-|=2\sqrt{\frac{3\ln\frac{40\ln 2N}{\delta}}{N}}$. Thus, $|\langle \theta^\star, 2h\rangle|\le 4\sqrt{\frac{3\ln\frac{40\ln 2N}{\delta}}{N}}$. If the total number of samples taken during this phase is $n=O(N)$, this means that $n|r(h,-h)|\le\tilde O\left(\sqrt{T}\ln\frac{1}{\delta}\right)$. Now, if we could show that the regret of $C_+$ with respect to $h$ and $C_-$ with respect to $-h$ were both also $\tilde O\left(\sqrt{T}\ln\frac{1}{\delta}\right)$, this would establish our regret bound with respect to the hint during this phase (because $r(h,a)=r(-h,a) +r(h,-h)$ for all $a$). To this end, consider two cases: $\|P_h^\perp\theta^\star\|^2 \ge d\|\theta^\star\|/\sqrt{T}$ or not. If $\|P_h^\perp\theta^\star\|^2\ge d\|\theta^\star\|/\sqrt{T}$, then by \pref{lem:estimatenormhp2e}, we have that $C_+$ and $C_-$ have regret with respect to $h$ and $-h$ of:
\begin{align*}
    \tilde O\left( \frac{d\|\theta^\star\|}{\|P^\perp_h\theta^\star\|^2}\ln\frac{1}{\delta} + \frac{\sqrt{d}}{\Delta \|P^\perp_h\theta^\star\|}\ln\frac{1}{\delta}\right)=\tilde O\left(\sqrt{T}\ln\frac{1}{\delta}\right)
\end{align*}
Alternatively, if $\|P_h^\perp\theta^\star\|^2\le d\|\theta^\star\|/\sqrt{T}$, then again by \pref{lem:estimatenormhp2e}, the same regret values are bounded by:
\begin{align*}
    O\left(T\|\theta^\star\|\Delta^2  + T\frac{\Delta\|P^\perp_h\theta^\star\|}{\sqrt{d}}\right)\le O(\sqrt{T})
\end{align*}
This establishes the desired regret bounds with respect to the hint.

To establish the bounds with respect to the optimal action, notice that \pref{lem:estimatenormhp2e}, implies that $C_+$ achieves regret:
\begin{align*}
    \tilde O\left[\min\left(T \frac{\|P_h^\perp \theta^\star\|^2}{\|\theta^\star\|} + T\|\theta^\star\| \Delta^2, \ 
    \frac{d(1 + \|h \|)}{\Delta^2 \|  \theta^\star\|}\ln \frac{1}{\delta}
    + 
    \frac{d(1 + \|h \|)  \| \theta^\star\|}{\| P_h^\perp \theta^\star\|^2 }\ln \frac{1}{\delta}
    \right)\right]
\end{align*}
Now by once considering two cases depending on whether $\|P_h^\perp\theta^\star\|\ge d\|\theta^\star\|/\sqrt{T}$, we see that this result implies a total regret (for $C_+$) of $\tilde O(d\sqrt{T}\ln\frac{1}{\delta})$.

Now, let us tackle the regret of $C_-$. We again consider two cases, either $\langle \theta^\star ,h\rangle\le \|\theta^\star\|/4$ or not. If $\langle \theta^\star, h\rangle \le \|\theta^\star\|/4$, then $\langle \theta^\star,-h\rangle \ge -\|\theta^\star\|/4$ and so the last conclusion of \pref{lem:estimatenormhp2e} applies to $C_-$ as well so that the same argument as in the previous paragraph shows that the regret of $C_-$ with respect to the optimal action is $\tilde O\left(d\sqrt{T}\ln \frac{1}{\delta}\right)$. Alternatively, if $\langle \theta^\star, h\rangle\ge \|\theta^\star\|/4$, notice that since $|\langle \theta^\star, 2h\rangle|\le 4\sqrt{\frac{3\ln\frac{40\ln 2N}{\delta}}{N}}$, we must have 
\begin{align*}
    N&\le \tilde O\left( \frac{\ln\frac{1}{\delta}}{|\langle \theta^\star, h\rangle|}\right)\le \tilde O\left(\frac{\ln \frac{1}{\delta}}{ \|\theta^\star\|}\right)\le \tilde O\left(\frac{\ln \frac{1}{\delta}\sqrt{T}}{d}\right)
\end{align*}
Thus the regret obtained by $C_-$ cannot be more than $\tilde O\left(\frac{\sqrt{T}}{d}\ln \frac{1}{\delta} \right)$.

This completes the analysis of Phase 2 when at most one instance returns. Let's now consider the case: both instances return, but the first one to return provides a value of $r_\perp$ that such that $\frac{0.06r_\perp^2}{2\cdot 5^2}\cdot \frac{r_\perp^2}{r}\le W d\log(T)/\sqrt{T}$. For this case, notice that since $\langle \theta^\star,h\rangle \ge \langle \theta^\star -h\rangle$, under our assumed events $E_{Y,\pm}$, we will never eliminate $C_+$ in line 12 of the algorithm. Thus since both instances return, it must be that the first-returning instance was $C_-$. Now, after returning, $C_-$ clearly incurs zero additional regret with respect to $-h$. Thus, by the same argument as in the previous case, the $Y_+\cup Y_-=\varnothing$ test will trigger before the regret with respect to the hint exceeds $\tilde O\left(\sqrt{T} \ln\frac{1}{\delta}\right)$. For the worst-case regret, notice that $r_h \le \tilde O(\ln(1/\delta)/\sqrt{T})$ (because the return value for $C_-$ did not trigger the test in line 11), we have that the worst-case regret is the regret with respect to $h$ plus $Tr_h\le \tilde O\left(\sqrt{T}\ln\frac{1}{\delta}\right)$.

This completes the analysis of Phase 2.

\textbf{Phase 3:} Now, we focus on the last phase. At this point, we have established that $r_\perp$ is a constant-factor approximation of $\|P^\perp_h\theta^\star\|$ and $r$ is a constant-factor approximation of $\|\theta^\star\|$. Therefore $r_\perp^2/r=\Theta(\|P^\perp_h\theta^\star\|^2/\|\theta^\star\|)=\Theta(r_h)$, where the final equality follows from \pref{lem:regret_unitball}. Moreover, since $r\ge 0.06\|\theta^\star\|$ and $r_\perp<5\|P^\perp_h\theta^\star\|$, $\frac{0.06r_\perp^2}{2\cdot 5^2}\cdot \frac{r_\perp^2}{r}\le \frac{1}{2}\cdot\|P^\perp_h\theta^\star\|^2/\|\theta^\star\|\le r_h$. Further, since we cannot eliminate $h$, we must have that either $|S|=1$ and $h$ is the remaining hint, or $|S|=2$, but $r_\perp$ is such that \textsc{Switch} would not choose to play the hint in any event. Thus \textsc{Switch} incurs no regret with respect to $h$, while always maintaining a regret of $O(d\sqrt{T})$ with respect to the optimal action.
\end{proof}

\subsection{Algorithms and proofs for Pareto and Multi-Hint settings}

\begin{algorithm2e}
\SetAlgoVlined
\SetKwInOut{Input}{Input}
\SetKwProg{myproc}{Procedure}{}{}

\DontPrintSemicolon
\LinesNumbered
\newcommand\mycommfont[1]{\footnotesize\textcolor{DarkBlue}{#1}}
\SetCommentSty{mycommfont}
\Input{hint $h \in \R^d$, number of rounds $T$, failure probability $\delta$, target total hint regret: $G$}
\tcp{Phase 1: Estimate norm $\|\theta^\star\|$}
Initialize $C_0 \gets
\textsc{EstimateNormHP}(0, 1, \frac{\delta}{4})$\;
Call $C_0.$\textsc{PlayAndUpdate}$()$ until it returns a value $r$\;

\tcp{Phase 2: Estimate norm of orthogonal complement $\|P^\perp_h \theta^\star\|$}
Set exploration radius $\Delta = \frac{\sqrt{G}}{\sqrt{rT}}$\;
Initialize $C_+ \gets \textsc{EstimateNormHP}(+h, \Delta, \frac{\delta}{4})$ and $C_- \gets \textsc{EstimateNormHP}(-h, \Delta, \frac{\delta}{4})$\;
Initialize active set $\mathcal S = \{ C_+, C_-\}$\;

\tcp{Elements of $S$ are arms, a pull corresponds calling \texttt{PlayAndUpdate}, reward is the unperturbed hint loss}

Run \texttt{MultiArmBandit(S, G)} until any instance in $S$ returns a value $r_\perp$ satisfying $r_\perp^2/r \geq c_0 *  d\log(T)/G$ or $|S| = 1$ and the lone instance returns $r_\perp$

\tcp{Phase 3: Commit to hint or ignore it. Note $c_0 > c_1$.}
For remaining rounds, call $\textsc{Switch}(h, \nicefrac{r_\perp^2}{r}, T, c_1 *  d\log(T)T/G)$ for some $h \in S$ randomly chosen

\caption{Pareto Frontier: Bandit Algorithm on Unit Ball}
\label{alg:unit_ball_pareto2e}
\end{algorithm2e}

\begin{proof}[Proof of \pref{lem:prior_regret_pareto}]
This lemma follows similarly to Lemma~\ref{lem:regretlowprob}.

By Lemma~\ref{lem:prior_regret_perturb}, we see that our hint-based regret is bounded by high constant probability by:

$$ T_e * \|\theta^\star\|  \left(\Delta^2 + \Delta \frac{\|P_h^\perp a^\star\|}{\sqrt{d}}\right) $$

where $T_e$ is the number of iterations of $\texttt{EstimateNorm}$. Let us write $\theta^\star = \alpha h + v$, then note $\|P_h^\perp a^\star\| = \|v\|/\|\theta^\star\|$ , so if $\|v\|^2 \leq \|\theta^\star\|^2  d\log(T)\Delta^2$, then we have that 

$$R_h^T \leq O( \|\theta^\star\| \Delta^2 T\sqrt{\log(T)}) $$

Furthermore, in this case, note that the instantaneous regret of playing $h$ is, by Lemma~\ref{lem:regret_unitball}, given by $\Theta(\|v\|^2/\|\theta^\star\|) = \|\theta^\star\|d \log(T) \Delta^2$. Therefore, the full regret is bounded by $O(d\|\theta^\star\|\Delta^2 T\log(T))$.

Otherwise, we have $\|v\|^2 \geq \|\theta^\star\|^2  d\log(T)\Delta^2$. Then by Lemma~\ref{lem:estimatenormhp2e}, the maximum rounds of iterations is $T_e = O( \frac{d\log(T)}{\|v\|^2 \Delta^2})$, with probability at least $2/3$, so we can bound our hint-based regret by

$$R_h^T = O\left (\frac{d\log(T)}{\|v\|^2 \Delta^2}\right) * \|\theta^\star\|  \left(\Delta^2 + \Delta \frac{\|v\|}{\|\theta^\star\|\sqrt{d}}\right) = O\left( \frac{ \log(T)}{\|\theta^\star\| \Delta^2 } \right)$$

For the worst case regret, if $\langle h, \theta^\star\rangle \geq -\|\theta^\star\|/4$, then we incur regret at most 

$$O
\left(\frac{d\|\theta^\star\|\log(T)}{\|v\|^2} + \frac{d\log(T)}{\Delta^2\|\theta^\star\|}\right) = O\left(\frac{d\log(T)}{\Delta^2\|\theta^\star\|}\right)$$

by our Lemma~\ref{lem:estimatenormhp2e}.

\end{proof}

\begin{proof}[Proof of \pref{thm:main_single_prior_pareto}]
This proof follows directly from combining Lemma~\ref{lem:prior_regret_pareto} with the same reasoning as Theorem~\ref{thm:main_single_prior_2e}. We borrow the same notation and split our regret analysis into 3 phases, as labeled in the psuedocode: in the first phase we estimate $\|\theta^\star\|$, in the second phase we estimate $\|P^\perp_h \theta^\star\|$, and finally in the last phase we call $\textsc{Switch}$. Note that our algorithm may not execute all three segments before we reach our iteration budget $T$.

\textbf{Phase 1:} For the first phase, since $\|\theta^\star\| \geq d/G$, by \pref{lem:estimatenormhp2e}, we obtain $r$ satisfying $0.06\|\theta^\star\|\le r\le 5\|\theta^\star\|$ after $\tilde O(d\ln(1/\delta)/\|\theta^\star\|^2)$ rounds, incurring $\tilde O\left(\frac{d}{\|\theta^\star\|}\ln\frac{1}{\delta}\right) = \tilde O(G\ln(1/\delta))$ regret. Since the total regret is bounded by $\tilde O(G\ln(1/\delta))$, the hint-based regret is also similarly bounded. 

\textbf{Phase 2:} For the second phase, we call Algorithm~\ref{alg:hp_estimation2e} on $h$ and $-h$ as instances $C_+$ and $C_-$. We need to verify the hint-based and worst case regrets. 

For the hint-based regret, we first consider the non-perturbed actions. In this case, we need to bound $R(\mathbf{h}, -\mathbf{h})$ in the 2-arm bandit game, we play the hint-based MAB algorithm with total hint regret $G$ against the arm corresponding to $h$ (i.e. $C_+$) and worst case regret $O(dT/G)$ (as described in the upper bounds in \cite{lattimore2015pareto}). If the MAB terminates early, we bound the hint-based regret by using the worst case regret. Note that if some $C$ returns $r_\perp$ with $r_\perp^2/r \geq c_0 * d\log(T)/(G\sqrt{T})$, then since the worst case bound is $R = O(d T\log(T)/G)$ for some constant and $r_\perp^2/r$ is a constant approximation to $r_h$, we can find $c_0$ such that $r_h T \geq R$, which shows that our hint-based regret must be negative.   

Now, it suffices to add the perturbations and note that by the additivity property of the regret, we can simply bound the hint-based regret of $C_-, C_+$ to each hint respectively. To do this, we use \pref{lem:prior_regret_pareto} with $\Delta^2 = G/rT$ to bound the hint-based regret by $O(\|\theta^\star\| \Delta^2 T\log(T)) = O(G\log(T))$ for $G \leq \sqrt{T}$. Similarly, we bound the worst case regret by the same lemma by $O(dT\log(T)/G)$.

\textbf{Phase 3:} Now, we focus on the last phase. At this point, we have established that $r_\perp$ is a constant-factor approximation of $\|P^\perp_h\theta^\star\|$ and $r$ is a constant-factor approximation of $\|\theta^\star\|$. Therefore $r_\perp^2/r=\Theta(\|P^\perp_h\theta^\star\|^2/\|\theta^\star\|)=\Theta(r_h)$, where the final equality follows from \pref{lem:regret_unitball}. Moreover, since $r\ge 0.06\|\theta^\star\|$ and $r_\perp<5\|P^\perp_h\theta^\star\|$, $\frac{0.06r_\perp^2}{2\cdot 5^2}\cdot \frac{r_\perp^2}{r}\le \frac{1}{2}\cdot\|P^\perp_h\theta^\star\|^2/\|\theta^\star\|\le r_h$. Further, since we cannot eliminate $h$, we must have that either $|S|=1$ and $h$ is the remaining hint, or $|S|=2$, but $r_\perp$ is such that \textsc{Switch} would not choose to play the hint in any event. Thus \textsc{Switch} incurs no regret with respect to $h$, while always maintaining a regret of $O(dT\log(T)/G)$ with respect to the optimal action.

\end{proof}

\begin{proof}[Proof of \pref{lem:multihint_worstcase}]
Let $\mathcal E$ be the event where the statement of \pref{lem:estimatenormhp2e} holds all instances $C_i$, where $\langle h_i, \theta^\star \rangle \in Y_i $ at all times for all $h_i \in \mathcal H$. 
By \pref{lem:time_uniform_hoeffding},  \pref{lem:estimatenormhp2e} and a union bound, the probability of $\mathcal E$ is at least $1 - \delta / 2$. In this event, the number of rounds of all instances $C_i$ of 
\textsc{EstimateNormHP} is bounded as
\begin{align*}
    n_i \leq \frac{T}{mB} \wedge \tilde O \left( \frac{d(1 + \Delta^2 + \|h_i \|^2)}{\Delta^2 \| P_{h_i}^\perp \theta^\star\|^2}
\right)
\end{align*}
and when $\langle h_i, \theta^\star \rangle \geq - \nicefrac{1}{4} \| \theta^\star \|$, the total regret of $C_i$ is bounded as
\begin{align*}
O\left(n_i \frac{\|P_h^\perp \theta^\star\|^2}{\|\theta^\star\|} + n_i\|\theta^\star\| \Delta^2\right)
\le 
\tilde O \left(\frac{d}{\Delta^2} + \frac{T\Delta^2}{mB}\right)~.
\end{align*}

Note that for any $h_i$ with $\langle h_i, \theta^\star \rangle < - \nicefrac{1}{4} \| \theta^\star \|$, there is $-h_i \in \mathcal H$ with
$\instregret(-h_i, h_i) \geq \nicefrac{1}{2} \| \theta^\star \| \geq \nicefrac{c_1}{2}$. Also, we know that $\langle h_i, \theta^\star \rangle \in Y_i $ for both $h_i, -h_i$ and the interval length of $Y_i$ is $O(\log(N/\delta)/\sqrt{N})$ after $N$ calls to $C_i^+, C_i^-$. Therefore, by the interval width of $Y_i$, we conclude that $r(-h_i, h_i) \leq c\frac{\log(N/\delta)}{\sqrt{N}}$. Therefore, we pull $C_i^+$ at most $N = O(\log(T))$ times before it is eliminated (since $-h_i$ is assumed in this case to have a higher reward).

Summing over all $m$ (or $2m$) priors, we get a final regret bound bound of $$\tilde O\left(\frac{dm}{\Delta^2} + \Delta^2 \frac{T}{B}\right)$$

as long as $d/\Delta^2 \geq \log(T)$.

Since $\Delta^2 = m/\sqrt{T}$, the first term in the expression to be bounded by $d\sqrt{T}$. For $m \leq d$, this setting of $\Delta^2$ also implies that $\Delta^2 T/B \leq dm/\Delta^2$ for any positive integral value of $B$.

For the last $T - T/B$ rounds, since we reduced to the 1-prior case, our worst case regret is $\tilde O(d\sqrt{T})$ by  Theorem~\ref{thm:main_single_prior_2e}.
\end{proof}

\begin{algorithm2e}
\SetAlgoVlined
\SetKwInOut{Input}{Input}
\SetKwProg{myproc}{Procedure}{}{}

\DontPrintSemicolon
\LinesNumbered
\newcommand\mycommfont[1]{\footnotesize\textcolor{DarkBlue}{#1}}
\SetCommentSty{mycommfont}
\Input{hints $\{h_i \in \R^d\}_{i=1}^m$, number of rounds $T$, failure probability $\delta$, exploration ratio $B$, worst case regret scaling $W$}

Set exploration radius $\Delta = \frac{\sqrt{m}}{T^{1/4}}$ and active set $\mathcal S = \varnothing$\;
\ForEach{hint $h_i$}{
Initialize $C_{i}^+ \gets \textsc{EstimateNormHP}(h_i, \Delta, \frac{\delta}{4m})$\;
Initialize $C_{i}^- \gets \textsc{EstimateNormHP}(-h_i, \Delta, \frac{\delta}{4m})$\;
Add $C_{i}^+, C_{i}^-$ to active set $\mathcal S$\;
}
\Repeat{$\frac{T}{mB}$ iterations or $|\mathcal S| = 1$}{
\ForEach{active instance $C_i \in \mathcal S$}{
 Call $C_i.$\textsc{PlayAndUpdate}$()$\;
 \tcp{Maintain CI of hint's expected reward}
  $\mathcal R_i \gets$ all reward samples obtained by $C_i$ so far playing unperturbed hint\;
  Compute confidence interval $Y_i = (\bar y_i - b_i, \bar y_i + b_i)$ with $\bar y_i = \frac{1}{|\mathcal R_i|} \sum_{y \in \mathcal R_i} y$ and $b_i = \sqrt{\frac{3\ln(40m \ln(2|\mathcal R_i|) / \delta)}{|\mathcal R_i|}}$\;
}
\tcp{Eliminate worse hint if possible}
\If{$Y_i \cap Y_j = \varnothing$ for any $i, j$}
{
Remove $C_i$ with smaller $\bar y_i$ from active set $\mathcal S$\;
}
\If{$|\mathcal S| > 1$ and $C_i$ has returned a value $r_\perp$ satisfying $\frac{r_\perp^2}{r} \geq \frac{c_0 Wd\log(T)}{\sqrt{T}}$}
{
Remove $C_i$ from $\mathcal S$
}
}
\tcp{Commit to single hint after $\frac{T}{mB}$ iterations}
Choose $h$ randomly from $\mathcal S$ and call $\textsc{ParetoBandit}(h, T, \delta, W)$

\caption{Multi-Hint Bandit Algorithm on Unit Ball}
\label{alg:unit_ball_multi2e}
\end{algorithm2e}

\begin{proof}[Proof of \pref{lem:multihint_hintbased}]
The proof follows extremely closely to Theorem~\ref{thm:main_single_prior_2e} and \pref{lem:multihint_worstcase}. We again condition on $\mathcal E$ be the event where the statement of \pref{lem:estimatenormhp2e} holds all instances $C_i$, where $\langle h_i, \theta^\star \rangle \in Y_i $ at all times for all $h_i \in \mathcal H$, which holds with probability at least $1-\delta$. For the first part of the algorithm, note that we are essentially want to bound the hint-based regret of playing multi-arm bandit with the $m$ hints for $T/B$ rounds. First consider the case when the best hint $h^\star$ was not eliminated and when we consider the regret when playing a non-perturbed action.

Then, we claim that for the first $T/B$ rounds,

$$ \sum_{j} r(h_{a_j}, h^\star)\leq \sqrt{m T/B}\log(T) $$

Let $s_i = r(h_i, h^\star)$, then since the intervals of $Y_i$ are shrinking like $O(1/\sqrt{N})$, we conclude that each suboptimal arm is pulled at most $O(\log(T)/s_i^2)$ times as long as $s_i \geq \sqrt{mB\log(T)/T}$ and would have been eliminated with high probability by the reward confidence interval comparison. Let $R$ be a regret threshold, then the total regret with respect to the best hint is $R (T/B) + \frac{m\log(T)}{R}$, where the first term of the regret captures the regret for all arms with $s_i \leq R$ and the second term captures that for $s_i \geq R$. By setting $R = \sqrt{m B \log(T)/T}$, we get our final regret bound of $\sqrt{mT/B}\log(T)$.

Therefore, we see that hint-based regret is bounded by $\tilde O(\sqrt{mT/B})$, if we do not play a perturbed action. However, since we are playing a perturbed action at each step with an orthogonal perturbation of $h_i$, the hint-based regret due to the perturbation is $\tilde O(\Delta^2 * (T/B) + \sum_i \Delta \|P_{h_{a_i}}^\perp \theta^\star\|/\sqrt{d}) $ by direct application of Lemma~\ref{lem:prior_regret_perturb} and using the fact that $r(a_i, h^\star) \leq r(h_i, h^\star) + r(a_i, h_i)$, where $a_i$ is the perturbed action for hint $h_i$.

Now as in \pref{thm:main_single_prior_2e}, consider two cases: $\|P_h^\perp\theta^\star\|^2 \ge d\|\theta^\star\|/\sqrt{T}$ or not. If $\|P_h^\perp\theta^\star\|^2\le md\|\theta^\star\|/\sqrt{T}$, then since $\Delta^2 = m/\sqrt{T}$, we directly deduce that the hint-based regret is $\tilde O ( \Delta^2 (T/B) + (m/\sqrt{T})(T/B)) = \tilde O ( (m/B) \sqrt{T})$. Otherwise, if $\|P_h^\perp\theta^\star\|^2\ge md\|\theta^\star\|/\sqrt{T}$, then by
\pref{lem:estimatenormhp2e}, we have that $C_+$ and $C_-$ have regret with respect to $h$ and $-h$ of:
\begin{align*}
    \tilde O\left( \frac{d\|\theta^\star\|}{\|P^\perp_h\theta^\star\|^2}\ln\frac{1}{\delta} + \frac{\sqrt{d}}{\Delta \|P^\perp_h\theta^\star\|}\ln\frac{1}{\delta}\right)=\tilde O\left(\frac{1}{m}\sqrt{T}\ln\frac{1}{\delta}\right)
\end{align*}

Since there are at most $m$ hints, the total hint-based regret is bounded by $\widetilde{O} (\sqrt{T}\log(T))$. Therefore, we conclude that the hint-based regret is dominated by $\widetilde{O}((m/B)\sqrt{T})$. 

For the second part of the algorithm when $h$ is chosen randomly from $\mathcal S$, note that since all hints with $s_i \geq \Omega(\sqrt{mB\log(T)/T})$ would have been eliminated, deferring to Theorem~\ref{thm:main_single_prior_2e}, the hint-based regret for the remaining rounds is at most $\tilde O(s_i T + \sqrt{T}) = \tilde O (\sqrt{mBT})$, with respect to $h^*$. 

Therefore, our total hint-based regret is

$$ \tilde O(\sqrt{mBT} + (m/B ) \sqrt{T})$$

Setting $B = m^{1/3}$ gives our result.

Finally, if the best hint has been eliminated, then it must be the case that we can find $c_0$ such that $r_\perp^2/r \geq c_0 \frac{d\log(m)\log(T)}{\sqrt{T}} $ implies that $r_h \geq R/T$ since $r_\perp^2/r = \Theta(r_h)$, where $R$ is our worst case bound from above, which shows that our hint-based regret must be negative.
\end{proof}

\section{Auxiliary Technical Lemmas}


\begin{lemma}\label{lem:gaussian_norm_bounds}
Let $G \sim \mathcal N(0, \nicefrac{\mathbf{I}_d}{d})$ be an isotropic Gaussian random variable and $v \in \mathbb R^d$ be arbitrary. Then the events
\begin{align*}
   & \left\{
    \|G\|^2 \leq 1 + 2 \sqrt{\frac{\ln(1 / \delta)}{d}} + 2 \frac{\ln(1 / \delta)}{d}
    \right\}  
    & &
    \left\{ \|G\|^2 \geq  1 - 2 \sqrt{\frac{\ln(1 / \delta)}{d}} 
    \right\}\\
    &\left\{ 
    \langle v, G\rangle^2 
    \leq  \frac{2\|v\|^2}{d}\min\left\{\Phi^{-1}(1 - \delta/2)^2, \frac{1}{2} + \ln \frac{1}{\delta} + \sqrt{\ln \frac{1}{\delta}}\right\}
    \right\}\\
    &\left\{ 
    \langle v, G\rangle^2 \geq  \frac{\|v\|^2}{d}\max \left\{\Phi^{-1}(1/2 + \delta / 2)^2, 1  - 2\sqrt{\ln \frac{1}{\delta}}\right\}
    \right\}
\end{align*}
each have probability at least $1 - \delta$.
\end{lemma}
\begin{proof}
We know that $X = d \|G\|^2$ is $\chi^2_d$-distributed and thus, by the tail bounds of \citet{laurent2000adaptive}, we have
\begin{align*}
    \mathbb P\left( X - d \geq 2 \sqrt{d \ln \frac{1}{\delta}} + 2 \ln \frac{1}{\delta}\right) &\leq \delta &\textrm{and} &&
    \mathbb P\left( X - d \leq -2 \sqrt{d \ln \frac{1}{\delta}} \right) &\leq \delta~.
\end{align*}
Rearranging the first event gives that for $\delta \leq \exp(-1/2)$
\begin{align*}
    \|G\|^2 
    \geq \frac{1}{d} + 4 \ln (1 / \delta)
    \geq 2 \sqrt{\ln(1 / \delta)} + 2 \ln(1 / \delta)
    \geq 2 \sqrt{\frac{\ln(1 / \delta)}{d}} + 2 \frac{\ln(1 / \delta)}{d}
\end{align*}
happens with probability at most $\delta$. Rearranging the second event gives
\begin{align*}
    \|G\|^2 \leq 1 - 2 \sqrt{\ln(1 / \delta)} \leq  1 - 2 \sqrt{\frac{\ln(1 / \delta)}{d}}
\end{align*}
happens with probability at most $\delta$.
Further, for any $v \in \mathbb R^d$, the distribution of $Y = d \langle v, G\rangle^2$ is $\chi^2_1$ since, w.l.o.g. $v = e_1$ and $\langle v, G\rangle = G_1 \sim \mathcal N(0, 1/d)$ and thus we can apply the tail bound above again to obtain the desired statement for the remaining events.
\end{proof}

\begin{lemma}\label{lem:time_uniform_hoeffding}
Let $(X_i)_{i \in \mathbb N}$ be a sequence of independent $\sigma$-sub-Gaussian random variables. Then with probability at least $1 - \delta$ for all $n \in \mathbb N$ jointly
\begin{align*}
   \left| \sum_{i = 1}^n X_i \right| \leq \sigma \sqrt{\frac{3 \ln \frac{4 \ln(2n)}{\delta}}{n}}~.
\end{align*}
\end{lemma}
\begin{proof}
Follows directly from Theorem~1 by \citet{howard2021time} (see their Equation~(2)).
\end{proof}

%
%

\end{document}